\documentclass[11pt]{article}

\usepackage{amsfonts}
\usepackage{xcolor}
\usepackage{amsthm}
\usepackage[normalem]{ulem}

\usepackage{amsmath}
\usepackage{amssymb}
\usepackage[numbers,sort&compress]{natbib}

\usepackage{multirow}
\usepackage{mathtools}
\usepackage{graphicx}
\usepackage{url}
\usepackage[ruled]{algorithm2e}
\PassOptionsToPackage{algo2e,ruled}{algorithm2e}
\usepackage{amsthm}
\usepackage[margin=1in]{geometry}
\usepackage{hyperref}
\usepackage[capitalize]{cleveref}
\hypersetup{colorlinks=true,linkcolor=blue,citecolor=blue}

\newcommand{\comment}[1]{}
\definecolor{darkgreen}{RGB}{0,100,0}      % classic dark green
\definecolor{forestgreen}{RGB}{40,180,40}  % very popular in papers
\usepackage[mathscr]{eucal}
\usepackage{thm-restate}

\newtheorem{theorem}{Theorem}
\newtheorem{lemma}[theorem]{Lemma}
  
    \newtheorem{definition}{Definition}
  \newtheorem{proposition}[theorem]{Proposition}
  \newtheorem{remark}[theorem]{Remark}
  \newtheorem{corollary}[theorem]{Corollary}

\usepackage{pgfplots}
\pgfplotsset{compat=1.18}
\usetikzlibrary{patterns}

\title{Distribution-Free Sequential Prediction with Abstentions}
\author{
  \textbf{Jialin Yu}\\
  Georgia Institute of Technology\\
  \small{\texttt{yujl1@gatech.edu}}
  \and 
   \textbf{Mo\"ise Blanchard}\\
  Georgia Institute of Technology\\
  \small{\texttt{mblanchard41@gatech.edu}
  }
}

\date{}
\usepackage{thmtools}
\usepackage{thm-restate}
\usepackage[mathscr]{eucal}
\usepackage{bbm}

\newcommand{\acks}[1]{\section*{Acknowledgments}#1}

\newcommand{\nonl}
{\renewcommand{\nl}{\let\nl\oldnl}}% Remove line number for one line

\DeclareMathOperator*{\argmax}{arg\,max}

\usepackage{latexsym,tikz,graphicx}
\usetikzlibrary{calc}
\usetikzlibrary{decorations.pathreplacing}
\usetikzlibrary{patterns}

\renewenvironment{proof}[1][]{\par\noindent{\bf Proof #1\ }}{\hfill$\blacksquare$\\[2mm]}

\begin{document}

%--------- CUSTOM COMMANDS ----------
% names
\newcommand{\trw}{\text{\small TRW}}
\newcommand{\maxcut}{\text{\small MAXCUT}}
\newcommand{\maxcsp}{\text{\small MAXCSP}}
\newcommand{\suol}{\text{SUOL}}
\newcommand{\wuol}{\text{WUOL}}
\newcommand{\crf}{\text{CRF}}
\newcommand{\sual}{\text{SUAL}}
\newcommand{\suil}{\text{SUIL}}
\newcommand{\fs}{\text{FS}}
\newcommand{\fmv}{{\text{FMV}}}
\newcommand{\smv}{{\text{SMV}}}
\newcommand{\wsmv}{{\text{WSMV}}}
\newcommand{\trwp}{\text{\small TRW}^\prime}
\newcommand{\alg}{\text{ALG}}
\newcommand{\rhos}{\rho^\star}
\newcommand{\brhos}{\brho^\star}
\newcommand{\bzero}{{\mathbf 0}}
\newcommand{\bs}{{\mathbf s}}
\newcommand{\bw}{{\mathbf w}}
\newcommand{\bws}{\bw^\star}
\newcommand{\ws}{w^\star}
\newcommand{\Prt}{{\mathsf {Part}}}
\newcommand{\Fs}{F^\star}

\newcommand{\Hs}{{\mathsf H} }

\newcommand{\hL}{\hat{L}}
\newcommand{\hU}{\hat{U}}
\newcommand{\hu}{\hat{u}}

\newcommand{\bu}{{\mathbf u}}
\newcommand{\ubf}{{\mathbf u}}
\newcommand{\hbu}{\hat{\bu}}

\newcommand{\primal}{\textbf{Primal}}
\newcommand{\dual}{\textbf{Dual}}

\newcommand{\Ptree}{{\sf P}^{\text{tree}}}
\newcommand{\bv}{{\mathbf v}}

%bold commands
\newcommand{\bq}{\boldsymbol q}

%random variables
\newcommand{\rvM}{\text{M}}

%\mathcal commands
\newcommand{\Acal}{\mathcal{A}}
\newcommand{\Bcal}{\mathcal{B}}
\newcommand{\Ccal}{\mathcal{C}}
\newcommand{\Dcal}{\mathcal{D}}
\newcommand{\Ecal}{\mathcal{E}}
\newcommand{\Fcal}{\mathcal{F}}
\newcommand{\Gcal}{\mathcal{G}}
\newcommand{\Hcal}{\mathcal{H}}
\newcommand{\Ical}{\mathcal{I}}
\newcommand{\Kcal}{\mathcal{K}}
\newcommand{\Lcal}{\mathcal{L}}
\newcommand{\Mcal}{\mathcal{M}}
\newcommand{\Ncal}{\mathcal{N}}
\newcommand{\Pcal}{\mathcal{P}}
\newcommand{\Scal}{\mathcal{S}}
\newcommand{\Tcal}{\mathcal{T}}
\newcommand{\Ucal}{\mathcal{U}}
\newcommand{\Vcal}{\mathcal{V}}
\newcommand{\Wcal}{\mathcal{W}}
\newcommand{\Xcal}{\mathcal{X}}
\newcommand{\Ycal}{\mathcal{Y}}
\newcommand{\Ocal}{\mathcal{O}}
\newcommand{\Qcal}{\mathcal{Q}}
\newcommand{\Rcal}{\mathcal{R}}

\newcommand{\brho}{\boldsymbol{\rho}}

%\mathbb commands
\newcommand{\Cbb}{\mathbb{C}}
\newcommand{\Ebb}{\mathbb{E}}
\newcommand{\Nbb}{\mathbb{N}}
\newcommand{\Pbb}{\mathbb{P}}
\newcommand{\Qbb}{\mathbb{Q}}
\newcommand{\Rbb}{\mathbb{R}}
\newcommand{\Sbb}{\mathbb{S}}
\newcommand{\Vbb}{\mathbb{V}}
\newcommand{\Wbb}{\mathbb{W}}
\newcommand{\Xbb}{\mathbb{X}}
\newcommand{\Ybb}{\mathbb{Y}}
\newcommand{\Zbb}{\mathbb{Z}}

\newcommand{\Rbbp}{\Rbb_+}

\newcommand{\bX}{{\mathbf X}}
\newcommand{\bx}{{\boldsymbol x}}

\newcommand{\btheta}{\boldsymbol{\theta}}

\newcommand{\Pb}{\mathbb{P}}

\newcommand{\hPhi}{\widehat{\Phi}}

%\hat commands
\newcommand{\Sigmah}{\widehat{\Sigma}}
\newcommand{\thetah}{\widehat{\theta}}

%Functional commands
\newcommand{\indep}{\perp \!\!\! \perp}
\newcommand{\notindep}{\not\!\perp\!\!\!\perp}

\newcommand{\one}{\mathbbm{1}}
\newcommand{\1}{\mathbbm{1}}%{{\rm 1}\kern-0.24em{\rm I}}
\newcommand{\aprx}{\alpha}

%Article commands
\newcommand{\ST}{\Tcal(\Gcal)}
\newcommand{\x}{\mathsf{x}}
\newcommand{\y}{\mathsf{y}}
\newcommand{\Ybf}{\textbf{Y}}
\newcommand{\smiddle}[1]{\;\middle#1\;}%middle with spaces before and after

%Edition commands
\definecolor{dark_red}{rgb}{0.2,0,0}
\newcommand{\detail}[1]{\textcolor{dark_red}{#1}}

\newcommand{\ds}[1]{{\color{red} #1}}
\newcommand{\rc}[1]{{\color{green} #1}}

\newcommand{\mb}[1]{\ensuremath{\boldsymbol{#1}}}

\newcommand{\metric}{\rho}
\newcommand{\proj}{\text{Proj}}

\newcommand{\paren}[1]{\left( #1 \right)}
\newcommand{\sqb}[1]{\left[ #1 \right]}
\newcommand{\set}[1]{\left\{ #1 \right\}}
\newcommand{\floor}[1]{\left\lfloor #1 \right\rfloor}
\newcommand{\ceil}[1]{\left\lceil #1 \right\rceil}
\newcommand{\abs}[1]{\left|#1\right|}
\newcommand{\norm}[1]{\left\|#1\right\|}

\newcommand{\todo}[1]{{\color{red} TODO: #1}}
\newcommand{\algoname}{\textsc{AbstainBoost} }

\maketitle
\thispagestyle{empty}

\begin{abstract}
We study a sequential prediction problem in which an adversary is allowed to inject arbitrarily many adversarial instances in a stream of i.i.d.\ instances, but at each round, the learner may also \emph{abstain} from making a prediction without incurring any penalty if the instance was indeed corrupted. This semi-adversarial setting naturally sits between the classical stochastic case with i.i.d.\ instances for which function classes with finite VC dimension are learnable; and the adversarial case with arbitrary instances, known to be significantly more restrictive. For this problem, \cite{goel2023adversarial} showed that, if the learner knows the distribution $\mu$  of clean samples in advance, learning can be achieved for all VC classes without restrictions on adversary corruptions. This is, however, a strong assumption in both theory and practice: a natural question is whether similar learning guarantees can be achieved without prior distributional knowledge, as is standard in classical learning frameworks (e.g., PAC learning or asymptotic consistency) and other non-i.i.d.\ models (e.g., smoothed online learning). 
We therefore focus on the distribution-free setting where $\mu$ is \emph{unknown} and propose an algorithm \algoname based on a boosting procedure of weak learners, which guarantees sublinear error for general VC classes in \emph{distribution-free} abstention learning for oblivious adversaries.
%In this setting, \cite{goel2023adversarial} proposed an algorithm that learns VC-1 hypothesis classes with sublinear errors/costs. 
These algorithms also enjoy similar guarantees for adaptive adversaries, for structured function classes including linear classifiers. These results are complemented with corresponding lower bounds, which reveal an interesting polynomial trade-off between misclassification error and number of erroneous abstentions.
\end{abstract}

\section{Introduction}

We study the classical online prediction problem in which a learner sequentially observes an instance $x_t\in\Xcal$ and makes a prediction
about a value $y_t$ before observing the true label. The learner’s
objective is to minimize the error of its predictions $\hat{y}_t$ compared to
the true value $y_t$. The problem has been extensively studied under two extreme scenarios: the \emph{stochastic} setting where instances $x_1,\ldots,x_T$ are assumed to be independently and identically distributed (i.i.d.) usually from an unknown distribution $\mu$ to the learner, and the general \emph{adversarial} setting where no assumptions are made on the instance generation process. In the stochastic setting, also known as PAC learning \citep{vapnik1974theory,valiant1984theory}, simple algorithms can provably learn the best prediction function within function classes $\Fcal$ of finite Vapnik-Chervonenkis (VC) dimension \citep{vapnik1971uniform,vapnik1974theory,valiant1984theory}, a combinatorial measure of complexity that is well controlled for many practical hypothesis classes, including linear classifiers. On the other hand, learning is severely limited for adversarial data. Specifically, learning in the adversarial setting requires the function class to have finite so-called Littlestone dimension \citep{littlestone1988learning,ben2009agnostic}, which does not hold even for threshold function classes in $[0,1]$.\footnote{The threshold function class is defined as $\Fcal=\{x\in[0,1]\mapsto \1[x\leq x_0]: x_0\in[0,1]\}$.}

This significant gap has sparked substantial interest in intermediate learning models which aim to relax the i.i.d.\ assumption while preserving strong learning guarantees. This includes distributionally-constrained adversaries \citep{rakhlin2011online,blanchard2025distributionally} and specifically the smooth adversarial model which imposes density ratio constraints on data distributions \citep{rakhlin2011online,haghtalab2024smoothed,block2022smoothed}. 
As an alternative model for non-i.i.d.\ data, another line of work focused on algorithms robust to \emph{adversarial corruptions} in which an adversary may corrupt training \citep{valiant1985learning,kearns1993learning,bshouty2002pac,awasthi2017power,gao2021instance,hanneke2022optimal,balcan2022robustly,shafahi2018poison,blum2021robust} or test data \citep{feige2018robust, attias2019improved, montasser2019vc, montasser2020reducing, montasser2021unknown, montasser2022generic, montasser2020halfspaces}.
A central challenge with adversarial corruptions is that achieving strong predictive performance typically requires imposing substantial restrictions on the adversary, such as limiting corruptions to a small fraction of the data or to specific classes of perturbations. 
To circumvent these limitations, allowing the learner to \emph{abstain} on difficult or out-of-training instances has emerged as a fruitful workaround \citep{chow2009optimum,bartlett2008classification,goldwasser2020beyond,kalai2021towards,cortes2016learning}. In practice, an abstention from the machine learning model could trigger a conservative default policy, for instance in the context of content moderation systems, or prompt human review/intervention in medical diagnoses or autonomous driving applications. In high-risk applications such as these, a cautious abstention is often significantly more desirable and cost-effective than a prediction error, and enables the model to selectively predict on instances for which it has high confidence---an approach sometimes referred to as ``reliable learning'' \citep{kanade2014distribution,kalai2021towards}. 

%Similarly, \cite{goel2023adversarial} proposed the \emph{abstention learning} framework that sits between the two extremes, where there is an adversary that injects adversarial (or out-of-distribution) examples in a stream of i.i.d. examples, and the learner is allowed to abstain from predicting on potentially adversarial examples. There are many examples in which the assumption in this framework is meaningful. For instance, when a self-driving car encounters weather conditions or unknown information signs outside of its training, it is better for the algorithm to abstain and hand over access to the driver instead of making a wrong decision.\footnote{This example comes from \cite{goel2023adversarial}.}  

\paragraph{Distribution-free sequential learning with abstentions.} We consider the sequential prediction model with abstentions introduced by \cite{goel2023adversarial} in which an adversary may inject arbitrary corruptions within the clean sequence of i.i.d.\ instances from a distribution $\mu$. In particular, the adversary may corrupt half (or all) of the data instances $x_1,\ldots,x_T$. On the other hand, the learner may abstain on instances which they suspect may be corrupted. For simplicity, we focus on the \emph{realizable binary classification} case where all instances are labeled by a fixed in-class function $f^\star:\Xcal\to\{0,1\}$. In other terms, we focus on a \emph{clean-label} model \citep{shafahi2018poison,blum2021robust}, where corrupted instances are still correctly labeled. 
In this context, the learner aims to always make correct predictions whenever it does not abstain, thereby ensuring reliability, while simultaneously avoiding abstentions on clean instances. Crucially, the learner incurs no penalty for abstaining on corrupted instances; in light of the fundamental limitations for adversarial learning even for realizable data \citep{littlestone1988learning}, this is essentially necessary to achieve meaningful learning guarantees when the data may contain a large fraction of corrupted examples. 

In this abstention sequential learning framework, \cite{goel2023adversarial} propose learning algorithms that guarantee both low misclassification error and few abstentions on clean instances for any function class $\Fcal$ with finite VC dimension, under the core assumption that the clean sample distribution $\mu$ is \emph{known} to the learner a priori. While these results provide strong theoretical evidence for the benefit of abstentions in the presence of adversarial attacks, the knowledge of $\mu$ is an important practical limitation. Indeed, assuming access to the underlying data distribution significantly departs from classical results in the PAC learning and consistency literature, in which simple algorithms such as Empirical Risk Minimization (ERM) or nearest-neighbor-based algorithms achieve desired learning guarantees without any prior distributional knowledge. Accordingly, recent work has aimed to remove distributional assumptions in other beyond-i.i.d.\ learning models, including smooth adversaries \citep{block2024performance,blanchard2025agnostic}. For the abstention learning problem itself, \cite{goel2023adversarial} show as a proof of concept that prior knowledge of the distribution $\mu$ is not required for certain function classes: namely axis-aligned boxes and those with VC dimension $1$. In a previous version of their concurrent work, \cite{edelman2026reliableabstention} also show that learning linear classifiers in $\Rbb^2$ also does not require distributional knowledge. This naturally leads to the following fundamental question for abstention sequential learning:
\newpage %
\begin{quote}
    \emph{Is learning with abstentions possible for general function classes (with finite VC dimension) without prior knowledge of the data distribution $\mu$?}
\end{quote}

\paragraph{Our contributions.}
Throughout, we refer to the number of erroneous abstentions (on non-corrupted instances) as the \emph{abstention error}. Our results reveal a trade-off between the misclassification error and the abstention error of learning algorithms. Towards characterizing the achievable misclassification/abstention error landscape for distribution-free sequential prediction with abstentions, we make the following contributions. Our key findings are summarized in \cref{fig:tradeoff results}.
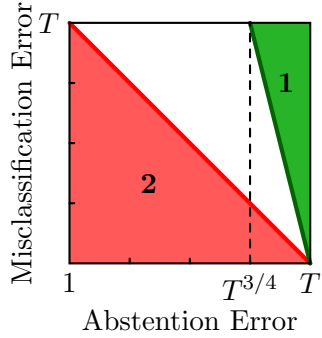
\begin{figure}[t] 
    \centering % 1. Centers the figure on the page
    
    \begin{tikzpicture}[scale=3.2, >=latex, line cap=round]
        % --- Definitions ---
        \def\xaxislabel{Abstention Error}
        \def\yaxislabel{Misclassification Error}
    
        % --- Regions ---
        
        % 1. Red Hatched Region 
        % \fill[pattern=north east lines, pattern color=red!70] 
        %     (0,0) -- (1,0) -- (0,1) -- cycle;
        \fill[red!65] (0,0) -- (1,0) -- (0,1) -- cycle;
        
        % 2. Green Hatched Region 
        % \fill[pattern=north east lines, pattern color=forestgreen] 
        %     (0.75, 1) -- (1, 1) -- (1, 0) -- cycle;
        \fill[forestgreen] 
            (0.66, 1) -- (1, 1) -- (1, 0) -- cycle;
        % --- Lines ---
        
        % Red Diagonal Line
        \draw[red, ultra thick] (0,1) -- (1,0);
        
        % Green Steep Line
        \draw[darkgreen, ultra thick] (0.66, 1) -- (1,0);
        
        % Dashed Vertical Line
        \draw[dashed, thick, black] (0.66, 0) -- (0.66, 1);
    
        % --- Axes and Box ---
        \draw[thick] (0,0) rectangle (1,1);
        
        % Ticks
        \foreach \x in {0.33, 0.66}
            \draw[thick] (\x, 0) -- (\x, 0.02);
        \foreach \y in {0.33, 0.66}
            \draw[thick] (0, \y) -- (0.02, \y);
    
        % --- Labels ---
        % --- Region Labels (Added) ---
        
        % Label "1" in Green Region (approximate centroid of upper right triangle)
        \node at (0.9, 0.75) {  \textbf{1}};
        
        % Label "2" in Red Region (approximate centroid of lower left triangle)
        \node at (0.33, 0.33) {  \textbf{2}};
        % Y-Axis Labels
        \node[left] at (0, 1) {  $T$};
        % Vertical Label
        \node[rotate=90] at (-0.2, 0.5) {\yaxislabel};
        
        % X-Axis Labels
        \node[below] at (0,0) {  $1$};
        \node[below] at (0.66, 0) {  $T^{2/3}$};
        \node[below] at (1, 0) { $T$};
        % Moved down slightly to -0.25 so it doesn't hit the numbers
        \node[below] at (0.5, -0.15) {\xaxislabel};
    
    \end{tikzpicture}
    
    % 2. Add Caption
    \caption{Tradeoffs between misclassification error and abstention error for either (1) oblivious adversaries or (2) adaptive adversaries and function classes with finite reduction dimension (see Definition~\ref{def:reduction_dimension}), including linear classifiers. The green region \textbf{1} is achievable (\cref{thm:main_oblivious_upper_bound,thm:main_adaptive_upper_bound}), while the red region \textbf{2} is not (\cref{thm:lowerbound}). 
    The plot is displayed in log-scale.} 
    \label{fig:tradeoff results}
\end{figure}

We propose a family of algorithms \algoname for $\alpha\in[0,1]$ which guarantee $\Ocal(T^{3\alpha})$ misclassification error and $\tilde\Ocal(d^2 T^{1-\alpha})$ abstention error for any function class with VC dimension $d$, against \emph{oblivious adversaries}.

Additionally, we show that \algoname also achieves similar learning guarantees against \emph{adaptive adversaries} when the function class has certain additional structure (finite so-called reduction-dimension, see \cref{def:reduction_dimension}). For instance, for linear classifiers in $\Rbb^p$, \algoname achieves $\Ocal(T^{3\alpha})$ misclassification error and $\tilde \Ocal(p^{4.67} T^{1-\alpha})$ abstention error\footnote{For the special case $p=2$, \cite{edelman2026reliableabstention} show that misclassification and abstention errors $\Ocal(T^{2/3})$ is achievable. In comparison, our bounds achieve $\Ocal(T^{3/4})$ combined error, however, this holds for all dimensions including $p>2$.}
This also recovers learnability against adaptive adversaries for VC-1 classes and the class of axis-aligned rectangles, provided by \cite{goel2023adversarial} with tailored algorithms, albeit with worse learning guarantees.

Our learning guarantees for oblivious and adaptive adversaries can be generalized to the \emph{censored} model, where the learner only observes the label $y_t=f^* (x_t)$ if it made a prediction at time $t$. The corresponding algorithm, \textsc{C-AbstainBoost}, enjoys nearly the same sublinear error guarantees as \algoname for both oblivious and adaptive adversaries. 

Last, we complement these results with an oblivious lower bound tradeoff between misclassification and abstention. Specifically, we show that for any $\alpha\in[0,1]$, learning algorithms must make either $\Omega(T^\alpha)$ misclassification error or $\Omega(T^{1-\alpha})$ abstention error against some oblivious adversary. As an important remark, the lower bound for $\alpha=1/2$ was known in a previous version of the concurrent work \citep{edelman2026reliableabstention}.

\comment{
Our algorithm is based on a reduction to the setting in which the learner would have access to few samples from $\mu$ before the learning procedure. Specifically, we consider weak learners which use instances observed at a few pre-specified times, then combine their predictions within a boosting procedure, of potentially independent interest. We note that usual learning-with-expert techniques do note readily apply since the boosting procedure needs to \emph{simultaneously} ensure low abstention excess error and low excess misclassification error compared to any good weak learner. Given a few samples from $\mu$, we use U-statistic estimators to achieve sufficient (exponential in $d$) precision for our purposes.
}

\paragraph{Organization.}
After describing the formal setup in \cref{sec:framework}, we detail our main results in \cref{sec:main results}. Next, in \cref{alg:weak_learner} we describe the weak learner strategy based on empirical estimates of shattering probabilities, and present the weak learner aggregation procedure in \cref{alg:main_exp_version}.Last, we extend our results to a censored model in which the learner does not receive feedback if it abstains in \cref{sec:censored_model} and conclude in \cref{sec:conclusion}.

\section{Preliminaries}\label{sec:framework}

\paragraph{Abstention learning setup.} Throughout this paper, we denote by $\Xcal$ the instance domain and let $\Fcal \subseteq \{0,1\}^{\Xcal}$ be a class of measurable functions. We consider the sequential online learning problem with adversarial corruptions,  where the learner may additionally abstain as introduced in \cite{goel2023adversarial}. 
Formally, the learning protocol proceeds as follows. Initially, an adversary fixes a distribution $\mu$ over $\Xcal$ together with labeling function $f^\star \in \Fcal$. 
The learner does not observe $f^\star$ nor $\mu$.
Then, at each iteration $t\in[T]$:

\begin{enumerate}
    \item The adversary first decides whether or not to corrupt this round by selecting $c_t\in\{0,1\}$.
    \item If $c_t=1$, the adversary selects an arbitrary instance $x_t\in\Xcal$. Otherwise, if $c_t=0$, the new instance is sampled $x_t\sim\mu$ independently from the past. 
    \item The learner observes $x_t$ and outputs $\hat y_t\in\{0,1,\perp\}$, where $\bot$ denotes an abstention.
    \item Finally, the learner observes the (clean) label $y_t = f^\star(x_t)$.
\end{enumerate}

We emphasize that this setup corresponds to a realizable setting: the values $y_t$ are always consistent with the fixed function $f^\star\in\Fcal$ even on corrupted rounds. 
In words, we consider \emph{mild} corruptions which only affect the instance $x_t$ but not the labelings.

\begin{remark}
    As discussed within \cite{goel2023adversarial}, other models of feedback are possible. For instance, we will also consider a \emph{censored} variant of this setup in \cref{sec:censored_model}, where the learner only observes label $y_t$ in step 4 when it has made a prediction i.e. $\hat{y}_t\neq \perp$.
\end{remark}

If the learner makes a prediction $\hat y_t\in\{0,1\}$, it incurs a misclassification cost if it makes a mistake: $ \hat y_t\neq y_t$. Conversely, if the learner abstains $\hat y_t=\perp$, it incurs an abstention cost if the round was not corrupted: $c_t= 0$. We therefore focus on the misclassification error and abstention error on all $T$ iterations, defined by
\begin{equation*}
    \textsc{MisErr}\coloneqq \sum_{t=1}^T \1\!\left[\hat y_t\notin \{y_t,\perp\}\right] \quad\text{and}\quad \textsc{AbsErr} \coloneqq \sum_{t=1}^T \1\!\left[c_t=0 \wedge \hat{y}_t=\perp\right].
\end{equation*}
In words, the learner aims to make few mistakes on the rounds where it predicts to minimize its misclassification error, while not abstaining too often on non-adversarial rounds to minimize its abstention error. Importantly, by construction, the learner is allowed to abstain on adversarially injected inputs for free. This allows for an arbitrary number of corruption rounds (e.g., $T/2$ corruptions) without incurring a linear penalty in the objective, as opposed to the abstention model \citep{Chow1957}, where there is a small but fixed cost on every abstention.

\comment{
\begin{algorithm}[htb]
\caption{Distribution-Free Abstention Learning}
\label{alg:protocol_abstention}
\DontPrintSemicolon
Adversary fixes an unknown distribution $\mu$ over $\Xcal$ and a labeling function $f^\star \in \Fcal$

\For{$t = 1,\dots,T$}{
  Adversary decides whether to inject an adversarial instance ($c_t=1$) or not ($c_t=0$)
  
  \lIf{$c_t=1$}{
    Adversary selects any instance $x_t \in \Xcal$
  }
  \lElse{
    $x_t \sim \mu$ is sampled independently from the history
  }
  
  Learner observes $x_t$ and outputs $\hat{y}_t \in \{0,1,\perp\}$ (where $\perp$ denotes abstention)
  
  Learner observes the label $y_t = f^\star(\hat{x}_t)$.\;
}
\end{algorithm}
}

\paragraph{Oblivious and adaptive adversaries.} 
We consider two different adversary models: \emph{oblivious} adversaries which fix in advance their corrupted instances, and \emph{adaptive} adversaries which may select corruptions adaptively during the learning procedure.

\begin{definition}[Oblivious and adaptive adversaries]
    An adversary is \emph{oblivious} if it fixes in advance corruption times $\{t: c_t=1\}$ together with their corrupted instances $(x_t)_{t:c_t=1}$, without adapting to the learner's output or clean instance samples.
    
    An adversary is \emph{adaptive} if its decision at time $t\in[T]$ may depend on all previous learner's outputs and clean instance samples.
\end{definition}

\paragraph{Complexity notion for the function class.}
In the sequential problem without abstentions---when the instances are i.i.d., known as the PAC learning setting---learnable function classes exactly correspond to those with finite VC dimension \citep{vapnik1971uniform,vapnik1974theory,valiant1984theory}.

\begin{definition}[VC dimension] Let $\Fcal:\Xcal\to\{0,1\}$ be a function class on $\Xcal$. We say that $\Fcal$ \emph{shatters} a set of points $x_1,...,x_m\in\Xcal$ if
for any choice of labels $\epsilon\in\{0,1\}^m$ there exists $f\in\Fcal$ such that $f(x_i)=\epsilon_i$ for all $i\in[m]$. The VC dimension of $\Fcal$ is the size of the
largest shattered set.
\end{definition}

Accordingly, we aim to achieve learning guarantees in the abstention learning setting when the function class $\Fcal$ has finite VC dimension.

\paragraph{Further notation.} 
Next, we introduce the notion of \emph{shattering probability}, which is a central quantity appearing in the algorithm developed by \cite{goel2023adversarial} in the case when the distribution $\mu$ of uncorrupted samples is known. Formally, the \emph{$k$-shattering probability} of a given function class $\Fcal:\Xcal\to\{0,1\}$ corresponds to the probability of $k$ i.i.d.\ points sampled from $\mu$ being shattered:
\begin{equation*}
    \rho_k(\Fcal,\mu) := \Pbb_{x_1,\ldots,x_k\overset{iid}{\sim}\mu} [\{x_1,\ldots,x_k\} \text{ is shattered by }\Fcal].
\end{equation*}
For instance, $\rho_1 (\Fcal,\mu)$ is the probability of an uncorrupted sample $x\sim \mu$ that can be labeled by both $0$ and $1$ by functions in $\Fcal$. Equivalently, $\rho_1 (\Fcal,\mu)$ is the probability of this sample falling in the \emph{disagreement region} $D(\Fcal) :=
    \set{x\in\Xcal:\exists f,g\in\Fcal, f(x)\neq g(x)}$.
Our algorithm uses observed data points to learn and reduce the hypothesis class. Given a dataset $A=\{(x_i,y_i)\}_i$, define the \emph{reduction} of $\Fcal$ to $A$ as
\begin{equation*}
    \Fcal\cap A := \{ f \in \Fcal : \forall i,\; f(x_i)=y_i\}.
\end{equation*}
When $A$ consists of a single labeled example $(x,y)$, we write $\Fcal^{x\to y}$ as shorthand for $\Fcal\cap{\{(x,y)\}}$. 
Last, by convention $\min\emptyset=+\infty$.

\section{Main Results}\label{sec:main results}

Our main result for oblivious adversaries is that sublinear misclassification and abstention errors can be achieved by an algorithm \algoname, which combines weak learners using a boosting abstention strategy. Its description is deferred to \cref{subsec:overview_algorithm}.

\begin{theorem}\label{thm:main_oblivious_upper_bound}
    Let $\Fcal$ be a function class with VC dimension $d$ and a horizon $T\geq 1$. Then, for any $\alpha\in[0,1/3]$, there is a choice of parameters for \algoname that achieves the following learning guarantee against any oblivious adversary:
    \begin{equation*}
        \textsc{MisErr} \lesssim T^{3\alpha} \quad \text{and} \quad \Ebb[\textsc{AbsErr}] \lesssim d^2 \log^{5/3}(T) \cdot T^{1-\alpha}.
    \end{equation*}
\end{theorem}

In particular, the family of algorithms \algoname provide a polynomial tradeoff between the abstention and misclassification errors, as depicted in the green region in \cref{fig:tradeoff results}. 

Naturally, one may wonder if such a polynomial tradeoff is necessary---for instance, can we achieve $\text{poly}(d\log T)$ misclassification error with $T^{1-\Omega(1)}$ abstention error, or inversely? \cite{edelman2026reliableabstention} showed that any algorithm must incur either $\Omega(\sqrt T)$ misclassification or $\Omega(\sqrt T)$ abstention error. We complete this result with the following lower bound tradeoff between abstention and misclassification errors, which shows that polynomial tradeoffs are necessary even for learning function classes with VC dimension one.

\begin{restatable}{theorem}{lowerbound}\label{thm:lowerbound}
    There is a VC-1 function class $\Fcal$ such that the following holds. For any $T\geq 1$ and any learner, for any $\alpha\in[0,1]$, there exists an oblivious adversary on $\Fcal$ for which either
    \begin{equation*}
        \Ebb[\textsc{MisErr}]\geq T^\alpha/32 \quad\text{or} \quad \Ebb[\textsc{AbsErr}]\geq T^{1-\alpha}/2.
    \end{equation*}
\end{restatable}

This lower bound tradeoff is depicted in the red region within \cref{fig:tradeoff results}.
While this leaves some polynomial gap between the guarantees of \algoname from \cref{thm:main_oblivious_upper_bound}, for VC-1 classes, this gives a complete characterization of the misclassification/abstention error landscape:

\begin{remark}
    For the special case of VC-1 function classes, \cite{goel2023adversarial} provided a family of algorithms which tuned for some parameter $\alpha\in[0,1]$\footnote{using their Algorithm 2 with parameter $T^{1-\alpha}$} achieves $\Ebb[\textsc{MisErr}]\leq 2T^\alpha$ and $\Ebb[\textsc{AbsErr}]\leq T^{1-\alpha}\log T$. Together with \cref{thm:lowerbound}, this gives a tight characterization of achievable misclassification and abstention error for axis-aligned rectangles (see \cref{sec:reduction_dimension_examples} for a definition) or VC-1 classes, up to logarithmic factors, depicted in \cref{fig:VC_1_tradeoffs}.
\end{remark}

\begin{figure}[t] 
    \centering % 1. Centers the figure on the page
    
    \begin{tikzpicture}[scale=3.2, >=latex, line cap=round]
        
        % --- Definitions ---
        \def\xaxislabel{Abstention Error}
        \def\yaxislabel{Misclassification Error}
    
        % --- Regions ---
        
        % 1. Red Hatched Region 
        % \fill[pattern=north east lines, pattern color=red!70] 
        %     (0,0) -- (1,0) -- (0,1) -- cycle;
        \fill[red!65] (0,0) -- (1,0) -- (0,1) -- cycle;
            
        % 2. Green Hatched Region 
        % \fill[pattern=north east lines, pattern color=green] 
        %     (0, 1) -- (1, 1) -- (1, 0) -- cycle;
        \fill[forestgreen] (0, 1) -- (1, 1) -- (1, 0) -- cycle;
    
        % --- Lines ---
        
        % Red Diagonal Line
        \draw[red, ultra thick] (0,1) -- (1,0);
        
        % Green Steep Line
        % \draw[green, ultra thick] (0.75, 1) -- (1,0);
        
        % Dashed Vertical Line
        % \draw[dashed, thick, black] (0.75, 0) -- (0.75, 1);
    
        % --- Axes and Box ---
        \draw[thick] (0,0) rectangle (1,1);
        
        % Ticks
        \foreach \x in {0.25, 0.5, 0.75}
            \draw[thick] (\x, 0) -- (\x, 0.02);
        \foreach \y in {0.25, 0.5, 0.75}
            \draw[thick] (0, \y) -- (0.02, \y);
    
        % --- Labels ---
        % --- Region Labels (Added) ---
        
        % Label "1" in Green Region (approximate centroid of upper right triangle)
        \node at (0.66, 0.66) { \textbf{1}};
        
        % Label "2" in Red Region (approximate centroid of lower left triangle)
        \node at (0.33, 0.33) { \textbf{2}};
        % Y-Axis Labels
        \node[left] at (0, 1) { $T$};
        % Vertical Label
        \node[rotate=90] at (-0.2, 0.5) {\yaxislabel};
        
        % X-Axis Labels
        \node[below] at (0,0) { $1$};
        \node[below] at (0.5, 0) { $T^{1/2}$};
        % \node[below] at (0.75, 0) { $T^{3/4}$};
        \node[below] at (1, 0) { $T$};
        % Moved down slightly to -0.25 so it doesn't hit the numbers
        \node[below] at (0.5, -0.15) {\xaxislabel};
    
    \end{tikzpicture}
    
    % 2. Add Caption
    \caption{Tradeoffs between abstention error and misclassification error for VC-1 function classes and axis-aligned rectangles. The upper bound (region \textbf{1}) from \cite{goel2023adversarial} applies to general adaptive adversaries, and the lower bound (region \textbf{2}) shares the same lower bound result (\cref{thm:lowerbound}) with \cref{fig:tradeoff results}. The plot is displayed in log-scale.} 
    \label{fig:VC_1_tradeoffs}
\end{figure}
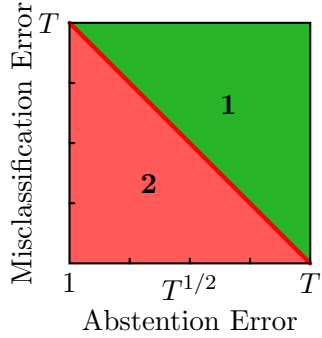

Last, we turn to the case of adaptive adversaries. At the high-level, this setting is significantly more challenging since the adversary may interfere with statistical estimations carried by the learner. Nevertheless, we show that \algoname can still achieve learning guarantees for useful function classes with additional structure. Specifically, in addition to having finite VC dimension, we will focus on function classes with finite so-called \emph{reduction dimension}, defined as follows:
\begin{definition}
\label{def:reduction_dimension}
    Let $l\geq 1$. We say that a function class $\Fcal:\Xcal\to\{0,1\}$ has $l$-reduction dimension $D_l$ if for any $n\geq D_l$, and any $x_1,\ldots,x_n\in\Xcal$:
    \begin{equation*}
        \abs{ \set{\Fcal\cap A|_{\{x_1,\ldots,x_n\}} : A\subseteq \Xcal\times\{0,1\}, |A|\leq l }} \leq n^{D_l}.
    \end{equation*}
\end{definition}
In words, a function class $\Fcal$ has finite $l$-reduction dimension if the number of possible function classes that can be obtained by reducing $\Fcal$ to functions consistent with $l$ labeled datapoints grows polynomially with the number $n$ of test points. This polynomial growth is reminiscent of the polynomial growth $\Ocal(n^d)$ of the number of different labelings induced by a function class with VC dimension $d$ on $n$ test points, known as Sauer-Shelah's lemma \citep{sauer1972density,shelah1972combinatorial}.

Of particular interest, linear classifiers in $\Rbb^d$---$\Fcal^d _{\text{lin}}:=\{x\in\Rbb^d \mapsto \1[a^\top x\geq b]: a\in\Rbb^d, b\in\Rbb\}$---have finite reduction dimension: we show in \cref{lemma:reduction_dim_linear_sep} that they have $l$-reduction dimension $D_l=\Ocal(d^2 l)$. This result uses links between oriented matroids and algebraic geometry together with Warren's theorem \citep{warren1968lower}, which bounds the number of topological components defined by polynomial equations. We refer to \cref{sec:reduction_dimension_examples} for further examples.

For function classes which have finite $\Ocal(d\log T)$-reduction dimension, we show that \algoname essentially achieves the same learning guarantee as in \cref{thm:main_oblivious_upper_bound} for adaptive adversaries.

\begin{theorem}\label{thm:main_adaptive_upper_bound}
    Let $\Fcal$ be a function class with VC dimension $d$ and $\ceil{5d^2\log T}$-reduction dimension $D$, and a horizon $T\geq 1$. Then, for any $\alpha\in[0,1/3]$, there is a choice of parameters for \algoname that achieves the following learning guarantees against any adaptive adversary:
    \begin{equation*}
        \textsc{MisErr} \lesssim T^{3\alpha} \quad \text{and} \quad  \Ebb[\textsc{AbsErr}] \lesssim d^{2} (D\log D+\log T)^{2/3} \log(T) \cdot T^{1-\alpha}.
    \end{equation*}
\end{theorem}
This shows that abstention learning is still possible for adaptive adversaries for structured VC function classes, with the same algorithmic ideas as for the oblivious case. For instance, for linear classifiers in $\Rbb^d$, \algoname achieves $\Ocal(T^{3\alpha})$ misclassification error and $\tilde \Ocal(d^{4.67} T^{1-\alpha})$ abstention error, and in particular, a $\tilde \Ocal(d^{4.67} T^{3/4})$ combined error for $\alpha=1/4$. In comparison, \cite{edelman2026reliableabstention} achieve $\tilde \Ocal(T^{2/3})$ combined error but only for the special case $d=2$. We believe the algorithmic ideas within \algoname can provide a useful starting point towards adaptively learning all VC classes:

\paragraph{Open question.} \textit{Can we achieve $\text{poly}(d)T^{1-\Omega(1)}$ misclassification and abstention error for any function class with VC dimension $d$, against adaptive adversaries?}

\vspace{5pt}

Finally, our results also naturally extend to the variant censored model in which the learner only observes the value $y_t$ at times when it makes a prediction. This require minor modifications of the algorithm to ensure that it can be run in the censored setting. We defer the description of this variant \textsc{C-AbstainBoost} to \cref{sec:censored_model}. It enjoys the same guarantee as \algoname for adaptive algorithms and slightly worse guarantees for the oblivious case for technical reasons.

\begin{theorem}\label{thm:censored_main}
    Let $\Fcal$ be a function class with VC dimension $d$ and a horizon $T\geq 1$. Then, in the censored model, for any $\alpha\in[0,1/3]$, there is a choice of parameters for \textsc{C-AbstainBoost} that achieves the following learning guarantee against any oblivious adversary:
    \begin{equation*}
        \textsc{MisErr} \lesssim T^{3\alpha} \quad \text{and} \quad \Ebb[\textsc{AbsErr}] \lesssim d^{10/3} \log^{7/3}(T) \cdot T^{1-\alpha},
    \end{equation*}
    and for another choice of parameters, the following guarantee against any adaptive adversary:
    \begin{equation*}
        \textsc{MisErr} \lesssim T^{3\alpha} \quad \text{and} \quad  \Ebb[\textsc{AbsErr}] \lesssim d^{2} (D\log D+\log T)^{2/3} \log(T) \cdot T^{1-\alpha},
    \end{equation*}
    where $D$ is the $\ceil{5d^2\log T}$-reduction dimension of $\Fcal$.
\end{theorem}

\subsection{Overview of our algorithmic approach}
\label{subsec:overview_algorithm}

When the distribution $\mu$ is known, the algorithm developed by \cite{goel2023adversarial} crucially relies on the knowledge of shattering probabilities which serve as a form of potential to decide when to abstain and which predictions to make---we refer to \cref{sec:weak learners} for an overview of their algorithm. 
Hence, the main challenge for designing algorithms without prior knowledge of $\mu$ is to estimate these shattering probabilities efficiently using only the observed sequence of instances, which may contain a large fraction of corruptions. To do so, we intuitively reduce the problem to the simpler setting in which the learner has access to a few clean samples from $\mu$, following two main steps.

\paragraph{Construction of weak learners.} For each subset $\Tcal\subseteq[T]$ of $N$ times, we construct a weak learner which uses the instances $\Scal=\{x_t: t\in\Tcal\}$ to estimate shattering probabilities. Specifically, this weak learner runs a variant of the algorithm from \cite{goel2023adversarial} which uses shattering probability estimators based on the samples in $\Scal$ only. For convenience, we also include as parameter for the weak learner, labelings $(z_t)_{t\in\Tcal}$ for all times in $\Tcal$. Together, this yields a weak learner $\textsc{WL}(\Tcal,z)$; we defer its detailed description to \cref{alg:weak_learner} in \cref{sec:weak learners}.

Importantly, when $\Tcal$ contains only uncorrupted iterations, these instances in $\Scal$ exactly correspond to $N$ i.i.d.\ samples from $\mu$. This essentially corresponds to the case when the learner is given $N$ uncorrupted samples from $\mu$ a priori, for which we obtain the following learning guarantee.
\begin{theorem}
\label{thm:iid_examples_v2}
There exists a universal constant $c_0>0$ such that the following holds. Fix $\epsilon\in(0,1)$. Let $\Fcal:\Xcal\to\{0,1\}$ be a function class of VC dimension $d$ and fix an adversary. Let $N= \ceil{2000d^2 /\epsilon}$.
Let $\Tcal$ denote the (potentially random) set of first $mN$ non-corrupted rounds and define $z:=(y_t)_{t\in\Tcal}$. 

If (1) the adversary is oblivious and $\texttt{update}=\texttt{always}$, $m \geq 8\log(dT/\epsilon)$, or (2) the adversary is adaptive and $\texttt{update}=\texttt{restricted}$, $m \geq c_0(D\log(D)+8D+3\log(d/\epsilon))$ where $D$ is the $\ceil{5d^2\log \frac{1}{\epsilon}}$-reduction dimension of $\Fcal$, then 
\begin{equation*}
    \textsc{MisErr}(\Tcal,z) < 5d^2\log\frac{1}{\epsilon} \quad \text{and} \quad \Ebb[\textsc{AbsErr}(\Tcal,z)] \leq 18\epsilon T,
\end{equation*}
where $\textsc{MisErr}(\Tcal,z)$ and $\textsc{AbsErr}(\Tcal,z)$ respectively denote the misclassification and abstention error of the weak learner $\textsc{WL}(\Tcal,z)$ as defined in \cref{alg:weak_learner}.
\end{theorem}

\paragraph{Boosting strategy to combine weak learners.} We then use a \textsc{Boosting} strategy to combine the recommendations of all such weak learners $\textsc{WL}(\Tcal,z)$ and achieve misclassification and abstention errors comparable to that of the best weak learner. We defer its detailed description to \cref{alg:main_exp_version} within \cref{sec:boosting}.
In fact, the boosting procedure does not require any specific property of the weak learners; we therefore present its guarantee in a general abstention learning framework with experts, of independent interest.
Consider the abstention online setup with $L$ experts in which at iteration $t\in[T]$: (1)
An adversary adaptively selects the recommendation $y_{i,t}\in\{0,1,\perp\}$ for each expert $i\in[L]$ together with the true value $y_t\in\{0,1\}$. (2) Next, the learner observes the recommendations $y_{i,t}$ for $i\in[L]$ then selects $\hat y_t\in\{0,1,\perp\}$. (3) At the end of the round, the true value $y_t$ is revealed to the learner.
We have the following guarantee on the \textsc{Boosting} procedure for this learning problem.

\begin{algorithm}[t]

\caption{\algoname}\label{alg:abstain_boost}

%\noindent\rule{\linewidth}{0.4pt}
\LinesNumbered
\everypar={\nl}
\SetAlgoNoEnd
\setcounter{AlgoLine}{0}

\KwIn{horizon $T$, precision $\epsilon\in[0,1]$, number of subsets $m$, subset size $N$, maximum per-layer deletion $s_{\max}\in[T]$, mistake tolerance $M$, update policy $\texttt{update}\in\{\texttt{always},\texttt{restricted}\}$ }

\vspace{3pt}

Run $\textsc{Boosting}$ (\cref{alg:main_exp_version}) with maximum per-layer deletion $s_{\max}$, mistake tolerance $M$, on all $L:=\sum_{k\leq mN}2^k\binom{T}{k}$ weak learners $\textsc{WL}(\Tcal,z)$ (\cref{alg:weak_learner}) run with precision $\epsilon$, update parameter $\texttt{update}$ and number of subsets $m$, for $\Tcal\subseteq[T]$ and $z=(z_t)_{t\in\Tcal}\in\{0,1\}^\Tcal$ with $|\Tcal|\leq mN$

\end{algorithm}

\begin{restatable}{theorem}{combiningweaklearners}\label{thm:combining_weak_lerners}
    Fix parameters $M\geq 0$ and $s_{\max}\in [T]$. Then, for $L$ experts, $\textsc{Boosting}_{s_{\max},M}$ achieves the following guarantee against any (adaptive) adversary: 
    \begin{align*}
        \textsc{MisErr}\leq c_0\frac{(M+\log T)T\log L}{s_{\max}} \quad\text{and}\quad   
        \textsc{AbsErr} \leq s_{\max}\ceil{\log L} + \min_{i\in[L]:M_i<M} A_i,
    \end{align*}
    for some universal constant $c_0>0$,
    where $M_i$ and $A_i$ respectively denote the misclassification and abstention error of expert $i\in[L]$.
\end{restatable}
We note that usual learning-with-expert techniques do not readily apply since the boosting procedure needs to \emph{simultaneously} ensure low excess abstention error and excess misclassification compared to good weak learners. Since \cref{thm:iid_examples_v2} precisely ensures the existence of a weak learner with less than $M=\Ocal(d^2\log T)$ misclassification error and low abstention error, 
running \textsc{Boosting} on all weak learners and adjusting parameters yields the desired learning guarantees for \cref{thm:main_oblivious_upper_bound} and \cref{thm:main_adaptive_upper_bound}. The final algorithm \algoname therefore simply performs the \textsc{Boosting} strategy over all weak learners for subsets $\Tcal\subseteq[T]$ and labelings $(z_t)_{t\in\Tcal}$ of bounded size $mN$ for some parameters $m,N\geq 1$. Its complete description is given in \cref{alg:abstain_boost}.

\section{Abstention learning with few uncorrupted samples}\label{sec:weak learners}

\paragraph{Algorithmic overview of \cite{goel2023adversarial}.} 
To gain intuition on the construction of weak learners, we first review the algorithm in \cite{goel2023adversarial}, detailed in \cref{alg:their_original_alg}, which crucially relies on prior knowledge of $\mu$ to estimate shattering probabilities $\rho_k(\cdot,\mu)$. Throughout, this section only, we denote by $\Fcal_t:=\Fcal\cap\{(x_s,y_s):s<t\}$ the reduced function class given prior datapoints at iteration $t\geq 1$. In particular, $\Fcal_t$ is non-increasing in $t$ and hence $k$-shattering probabilities $\rho_k(\Fcal_t,\mu)$ are also non-increasing in $t$ for all $k\in[d]$.

\begin{algorithm}[t]
\caption{Algorithm for abstention online learning with knowledge of the distribution $\mu$}\label{alg:their_original_alg}
% \SetAlgoLined
\LinesNumbered
\everypar={\nl}
Initialize $k\gets d$ and $\Fcal_{1}=\Fcal$

\For{$t=1,\ldots,T$ and while $k>1$}{
    Receive $x_t$

    \lIf{$\min\left\{\rho_k(\Fcal_t^{x_t\to 1},\mu),\ \rho_k(\Fcal_t^{x_t\to 0},\mu)\right\}\ge 0.6\,\rho_k(\Fcal_t,\mu)$}{output $\hat{y}_t=\perp$}
    \lElse{predict $\hat{y}_t=\argmax_{y\in\{0,1\}}\left\{\rho_k(\Fcal_t^{x_t\to y},\mu)\right\}$}
    Upon receiving label $y_t$, define $\Fcal_{t+1}= \Fcal_t^{x_t\to y_t}$
    
    \lIf{$\rho_k(\Fcal_{t+1},\mu)\le T^{-k}$}{Set $k\gets k-1$}
}
\lIf{$k=1 \text{ and } x_t\in D(\Fcal_t)$}{$\hat{y}_t=\perp$}
\lIf{$k=1 \text{ and } x_t\notin D(\Fcal_t)$}{Predict the (unique) consistent label for $\hat{x}_t$}
\end{algorithm}

The first observation is that whenever $\rho_1(\Fcal_t,\mu)\leq T^{-1}$, the algorithm only makes a prediction when the instance is fully disambiguated $x_t \notin D(\Fcal_t)$, incurring no misclassification errors (see lines 9-10 of \cref{alg:their_original_alg}).
Indeed, this results in a probability of a bad abstention at $t$ of exactly $\rho_1(\Fcal_t,\mu)\leq T^{-1}$ and hence a negligible overall abstention error. 
The goal of the algorithm is therefore to decrease the $1$-shattering probability when it makes a misclassification error, while also ensuring low abstention error. This is in fact possible if the $2$-shattering probability is small, say $\rho_2(\Fcal_t,\mu)\leq T^{-2}$. More generally, they show the following bound.

\begin{lemma}[Lemma 4.2 from \cite{goel2023adversarial}]
\label{lemma:goel_main_lemma}
    For any function class $\Fcal:\Xcal\to\{0,1\}$, distribution $\mu$ on $\Xcal$, and $k\geq 1$,
    \begin{equation*}
		\Pbb_{x\sim\mu}\sqb{\min\left\{ \rho_k (\Fcal^{x\to 0},\mu), \rho_k (\Fcal^{x\to 1},\mu) \right\} \geq 0.6 \rho_k(\Fcal,\mu) } \leq 5 \frac{\rho_{k+1}(\Fcal,\mu)}{\rho_{k}(\Fcal,\mu)}.
\end{equation*}
\end{lemma}

In particular, provided that $\rho_{k+1}(\Fcal,\mu)/\rho_{k}(\Fcal,\mu) \leq T^{-1}$, the algorithm may abstain whenever the instance $x_t$ falls into the bad event from \cref{lemma:goel_main_lemma} without incurring significant abstention cost, and otherwise predict the value $\hat y_t=y\in\{0,1\}$ yielding the larger value for $\rho_k (\Fcal^{x_t\to y},\mu)$ (see lines 4-5 of \cref{alg:their_original_alg}). 
This ensures a decrease of the $k$-shattering probability by a factor $0.6$ whenever a misclassification error occurs, leading to the following structure for \cref{alg:their_original_alg}. 
Since $\rho_{d+1}(\Fcal,\mu)=0$ ($\Fcal$ has VC dimension $d$) we can first decrease the $d$-shattering probability until $\rho_d(\Fcal_t,\mu)\leq T^{-d}$. Then, we can decrease the $(d-1)$-shattering probability until $\rho_{d-1}(\Fcal_t,\mu)\leq T^{-(d-1)}$; and continue successively until we decrease $\rho_1(\Fcal_t,\mu)\leq T^{-1}$
(see line 7 of \cref{alg:their_original_alg}).

\paragraph{Construction of the weak learner.} We now turn to our main setting of interest when the distribution $\mu$ is unknown.
As discussed within \cref{subsec:overview_algorithm}, we may consider in this part that the learner has access to a small set of clean i.i.d.\ samples from $\mu$, which can be used to estimate the shattering probabilities required by the $\mu$-dependent algorithm.\footnote{The algorithm would have similar learning guarantees if instead of knowing $\mu$, the learner has access to estimates of shattering probabilities accurate up to say a $(1\pm 0.05)$ factor.} The main difficulty is that it requires estimating shattering probabilities of order $T^{-\Omega(d)}$, i.e., exponentially small in the VC dimension; while we may have at most $T$ samples from $\mu$ even if no input was corrupted.

First, for any finite subset $S\subseteq\Xcal$---which we think of having i.i.d.\ samples from $\mu$---the $k$-shattering probability of a function class $\Fcal$ has the natural following U-statistic estimator
\begin{equation*}\label{eq:hat(rho_k) given S}
    \hat \rho_k^S(\Fcal) := \binom{|S|}{k} ^{-1} \sum_{S'\subseteq S,|S'|=k} \1[S'\text{ is shattered by }\Fcal].
\end{equation*}
For convenience, we denote by $\sigma_k ^N (\Fcal,\mu)$ the standard deviation of the random variable $\hat{\rho}_k^S(\Fcal)$ when $S$ contains $N$ i.i.d.\ samples from $\mu$.
To achieve Gaussian failure probabilities, we consider the median of these estimators: for a collection of subsets $\Scal=\{S_i,i\in[m]\}$, we define the $k$-shattering median estimator $$\rho_k^{\Scal}(\Fcal) := \text{Median}(\hat\rho_k^{S_1}(\Fcal),\ldots,\hat\rho_k^{S_m}(\Fcal)).$$

The central idea behind the construction of our weak learners is the following. The standard deviation $\sigma_k^N(\Fcal,\mu)$ when $S$ contains $N$ i.i.d.\ samples form $\mu$, can be bounded as a function of the $l$-shattering probabilities of $\Fcal$ for $l<k$ (see \cref{eq:variance_general_bound} in the appendix). Specifically, even though the $S$ contain at most $N\ll T$ samples, the estimator $\rho_k^{\Scal}(\Fcal)$ may have exponential precision provided that $l$-shattering probabilities for $l<k$ have an exponential decay.

\begin{restatable}{lemma}{calculationvariance}
\label{lemma:calculation_variance}
    Fix $\Fcal:\Xcal\to\{0,1\}$, a distribution $\mu$ on $\Xcal$, $k\geq 1$, and $N\geq k^2$. Let $\eta=\frac{k^2}{N}$. Suppose that for some $c\geq 1$, we have $\rho_l(\Fcal,\mu) \leq c \cdot \eta^l$ for all $l<k$. Then, $$\sigma_k^N(\Fcal,\mu)< \sqrt{3 c \cdot \eta^k \rho_k(\Fcal,\mu)}.$$
\end{restatable}

For instance, taking $c=1$ within \cref{lemma:calculation_variance}, if $\rho_l(\Fcal,\mu)\leq \eta^l$ for all $l<k$, then with high probability the estimator $\rho_k^\Scal(\Fcal)$, where $\Scal$ is a collection of $\Ocal(\log T)$ subsets of $N$ i.i.d.\ samples from $\mu$, accurately estimates the $k$-shattering probability up to a $(1\pm 0.05)$ factor whenever $\rho_l(\Fcal,\mu)\gtrsim \eta^k$.

This motivates the following strategy. At iteration $t$, instead of aiming to decrease the $k_t$-shattering probability according to a decreasing sequence for $k_t$ as in \cite{goel2023adversarial}, we select
\begin{equation}\label{eq:selection_rule_k_t}
    k_t:=\max\{k'\in[d]:\forall l\in [k'], \rho_{l}^\Scal(\Fcal_t) > \epsilon \cdot \rho_{l-1}^\Scal(\Fcal_t) \}.
\end{equation}
Intuitively, this ensures that the conditions required for \cref{lemma:calculation_variance} hold for $k=k_t-1$ and $c=\rho_{k_t-1} (\Fcal_t,\mu)$. This essentially shows that the estimated $k_t$-shattering probabilities $\rho_{k_t}(\tilde \Fcal,\mu)$ for $\tilde\Fcal\in\{\Fcal_t,\Fcal_t^{x_t\to 0},\Fcal_t^{x_t\to 1}\}$ are accurate as long as $\rho_{k_t} (\Fcal_t,\mu)\gtrsim \frac{k^2}{N}\cdot \rho_{k_t-1} (\Fcal_t,\mu)$, providing the desired guarantee for shattering probability estimates. In addition, by construction of $k_t$, we have $\frac{\rho_{k_t+1}^\Scal(\Fcal_t)}{\rho_{k_t}^\Scal(\Fcal_t)}\leq \epsilon$. Provided that these estimates are accurate, this shows that the probability of a bad abstention at time $t$ is bounded by $\Ocal(\epsilon)$ from \cref{lemma:goel_main_lemma}, yielding an overall abstention error $\Ocal(\epsilon T)$.

This yields our final weak learner $\textsc{WL}(\Tcal,z)$. Before iteration $\max\Tcal$, it considers that only rounds in $\Tcal$ are uncorrupted. After then, it uses the samples $\{x_{t'}:t'\in\Tcal\}$ to estimate $k_t$-shattering probabilities where $k_t$ is defined in \cref{eq:selection_rule_k_t}. The weak learner is summarized in \cref{alg:weak_learner}. As a remark, \cref{alg:weak_learner} accommodates for two types of updates for the restricted function classes $\Fcal_t$: setting the parameter $\texttt{update}=\texttt{always}$ corresponds to the update policy discussed here. The other alternative $\texttt{update}=\texttt{restricted}$ will only be used for adaptive adversaries.

\begin{algorithm}[t]

\caption{Weak Learner $\textsc{WL}(\Tcal,z)$}\label{alg:weak_learner}

\LinesNumbered
\everypar={\nl}
\setcounter{AlgoLine}{0}
%\SetAlgoNoEnd

\KwIn{precision $\epsilon\in[0,1]$, subset of times $\Tcal$ and labels $z=(z_t)_{t\in\Tcal}$, number of subsets $m$, update policy $\texttt{update}\in\{\texttt{always},\texttt{restricted}\}$ }

\vspace{3pt}

\lFor{$t\leq \max \Tcal$}{
    If $t\in\Tcal$, predict $\hat y_t=z_t$, otherwise abstain $\hat y_t=\perp$
}
% Let $\mathcal{S}= \{S_i\}_{i\in[m]}$ be a (nearly even) partition of $\{x_{t'}:t'\in\Tcal\}$ and initialize $\Fcal_t\gets \Fcal$
Let $\mathcal{S} = \{S_i\}_{i=1}^m$ be a partition of $\{x_{t'} : t' \in \Tcal\}$ such that $|S_1| = |S_2| = \cdots = |S_m| = \left\lfloor |\Tcal|/ m \right\rfloor$, discarding any remaining samples if necessary. Initialize $\Fcal_t\gets \Fcal$.

\For{$t=1+\max \Tcal,\ldots,T$\text{ and }$\rho_1^\Scal(\Fcal_t)>\epsilon$}{
    Observe $x_t$
    
    Let $k=\max\{k'\in[d]:\forall l\in [k'], \rho_{l}^\Scal(\Fcal_t) > \epsilon \cdot \rho_{l-1}^\Scal(\Fcal_t) \}$. Note that $\rho_0 ^{\Scal}(\mathcal{F}_t)=1$.
        
        \lIf{$\min_{y\in\{0,1\}} \rho_k^\Scal(\Fcal_t ^{x_t\to y}) \geq 0.9\rho_k^\Scal(\Fcal_t)$}{Abstain $\hat y_t=\perp$}
        \lElse{Predict $\hat y_t = \argmax_{y\in\{0,1\}} \rho_k^\Scal(\Fcal_t ^{x_t\to y})$}
        Observe response $y_t$

        \lIf{$\texttt{update}=\texttt{always}$}{Update $\Fcal_{t+1}= \Fcal_t^{x_t\to y_t}$}
        \lElseIf{$\texttt{update}=\texttt{restricted}$}{If $\hat{y}_t\notin \{y_t,\perp\}$, update $\Fcal_{t+1}= \Fcal_t^{x_t\to y_t}$. Otherwise, keep $\Fcal_{t+1}=\Fcal_t$
}
    }
    
\lIf{$\rho_1^\Scal(\Fcal_t)\leq \epsilon$\text{ and }$x_t\in D(\Fcal_t)$}{Abstain $\hat y_t=\perp$}
\lIf{$\rho_1^\Scal(\Fcal_t)\leq \epsilon$\text{ and }$x_t\notin D(\Fcal_t)$}{Predict consistent label $\hat y_t=f(x_t)$ for any $f\in\Fcal_t$}

\end{algorithm}

\paragraph{Oblivious vs. adaptive adversaries.} In the above exposition, we omitted a major technicality: the guarantee from \cref{lemma:calculation_variance} on shattering probability estimates for a function class $\tilde\Fcal$ applies to i.i.d.\ samples independent from $\tilde \Fcal$. However, \cref{alg:weak_learner} requires estimating shattering probabilities of various function classes $\{\Fcal_t,\Fcal_t^{x_t\to 0},\Fcal_t^{x_t\to 1}:t\in(\max\Tcal,T]\}$. 
For \emph{oblivious} adversaries, we can check that estimates are always accurate by applying a simple union bound: the function classes considered are independent from the i.i.d.\ samples $\Scal$ of $\mu$ used by the learner.

On the other hand, \emph{adaptive} adversaries may select instances $x_t$ so that the class $\Fcal_t^{x_t\to 0}$ is badly correlated with the samples $\Scal$ used by the weak learner. To resolve this issue, we note that the current strategy would work if the reduced function classes $\Fcal_t$ were only updated when mistakes are made (see line 10 of \cref{alg:weak_learner}). Further, the total number of mistakes is very small: at most $l-1:=d^2\log(1/\epsilon)$\footnote{Recall that shattering probabilities are reduced by a factor at each mistake}. Hence, all function classes that can potentially be encountered by the learner when running the weak learner with the update policy $\texttt{update}=\texttt{restricted}$,\footnote{We use a different update policy $\texttt{update}=\texttt{always}$ for oblivious adversaries to achieve stronger misclassification and abstention dependency in the VC dimension $d$ of the function class.} are of the form 
\begin{equation}\label{eq:form_function_classes}
    \Fcal\cap A,\quad A\subseteq\Xcal\times\{0,1\},\, |A|\leq l.
\end{equation}

Fortunately, we can show a uniform concentration result on shattering probability estimates for all these function classes if the $l$-reduction dimension of $\Fcal$ is finite. As a reminder, the $l$-reduction dimension precisely bounds the number of different function classes of the form \cref{eq:form_function_classes} when projected onto test samples.

\begin{lemma}\label{lemma:uniform_convergence_median}
    There exists a universal constant $c_0\geq 1$ such that the following holds. Fix $k,N\geq 1$. Let $l\geq 1$ and $\Fcal:\Xcal\to\{0,1\}$ be a function class with $l$-reduction dimension $D_l$. Let $\Scal:=\{S_i,i\in[m]\}$ be a collection of $m\geq c_0(D_l\log (N D_l) +\log \frac{1}{\delta})$ independent sets of $N$ i.i.d.\ samples from $\mu$. Then, with probability at least $1-\delta$,
    \begin{equation*}
        |\rho_k^\Scal(\Fcal\cap A) - \rho_k(\Fcal\cap A,\mu)| \leq 2 \sigma_k^N (\Fcal\cap A,\mu) ,\quad A\subseteq\Xcal\times\{0,1\}, \,|A|\leq l.
    \end{equation*}
\end{lemma}

\cref{lemma:uniform_convergence_median} precisely shows that shattering probabilities of all function classes that may be encountered can be uniformly accurately estimated, completing the last argument for \cref{thm:iid_examples_v2}.

\section{Boosting procedure}
\label{sec:boosting}

In this section, we describe our boosting procedure, which allows to combine the $L$ weak learners into a single learner with the misclassification and abstention error guarantees from \cref{thm:combining_weak_lerners}. Importantly, this procedure will achieve the desired guarantee even for adversarial input.
That is, in this section, we consider the online setup with $L$ weak learners (experts) in which at iteration $t\geq 1$:
\begin{enumerate}
    \item An adversary adaptively selects the recommendation $y_{i,t}\in\{0,1,\perp\}$ for each expert $i\in[L]$ together with the true value $y_t\in\{0,1\}$.
    \item The learner observes the recommendations $y_{i,t}$ for $i\in[L]$ then selects $\hat y_t\in\{0,1,\perp\}$.
    \item At the end of the round, the true value $y_t$ is revealed.
\end{enumerate}

The boosting strategy requires an error bound parameter $M$. Specifically, up to removing predictions of weak learners with large misclassification error, we may assume that all weak learners have at most $M$ misclassification error---for the abstention learning problem, we will use $M=\Ocal(d^2\log T)$ since \cref{thm:iid_examples_v2} guarantees the existence of a good weak learner with this misclassification error. With this property in mind, the high-level idea behind our procedure is as follows. Weak learners which do not abstain for at least $s\gg M$ rounds make at least $s-M$ correct predictions and as a result, have small mistake per prediction average rate $\leq M/(s-M)$. Therefore, assuming that we can focus on these weak learners with at least $s$ predictions, the fraction of incorrect predictions compared to correct predictions throughout all $T$ rounds is at most $\leq M/(s-M)\ll 1$. In turn, majority-vote-inspired algorithms will achieve low misclassification error without having to abstain frequently. Since we do not know a priori which weak learners make $s\gg M$ predictions, we instead use a deletion strategy: we ignore the first $s$ predictions of each algorithm for some adaptively chosen value of $s$, which precisely removes weak learners with at most $s$ predictions. 
In practice, to achieve low misclassifciation error, our complete boosting procedure needs to perform $\log L$ rounds of such deletions.

Before presenting the full multi-layer algorithm, we focus on a single layer. 
For any parameter $s\geq 0$ and minimum number of predictions $C$ per round, we consider the algorithm \textsc{Delete} which deletes the first $s$ predictions of each weak learner, then follows the majority vote $\hat y_t$ if at least $C$ weak learners still make a prediction within $\{0,1\}$, and otherwise abstains $\hat y_t=\perp$. The procedure is formally described in \cref{alg:1_layer_expert}.

\begin{algorithm}[t]

\caption{Algorithm for one layer $\textsc{Delete}_{s,C}$}\label{alg:1_layer_expert}
\LinesNumbered
\everypar={\nl}

\KwIn{number of experts $L$, minimum number of predictions $C$, deletion number $s\geq 0$}

\vspace{3mm}

\For{$t\geq 1$}{
    
    Observe $y_{i,t}$ for all experts $i\in[L]$

    For each $i\in[L]$, if $|\{t'<t: y_{i,t'}\neq\perp\}\}|< s$, delete prediction $\tilde y_{i,t}:= \perp$, otherwise $\tilde y_{i,t}:=y_{i,t}$

    \lIf{$|\{i\in[L]: \tilde y_{i,t}\neq \perp\}| \geq C$}{
        Predict $\hat y_t:=\argmax_{y\in\{0,1\}}|\{i\in[L]: \tilde y_{i,t}=y\}|$
    }
    \lElse{Abstain $\hat y_t=\perp$}

    %Return $(\tilde y_{i,t})_{t\in[L]}$
}
\end{algorithm}

This first algorithm is tailored for structured online adversaries of the following form.

\begin{definition}[Structured adversaries]
\label{def:adversary_structure}
    We say that an adversary includes at most $C$ predictions per round if for any defined round $t\geq 1$, at most $C$ weak learners make a prediction, i.e., $|\set{i\in[L]: y_{i,t}\neq \perp}|\leq C$.

    For any subset of times $\Ucal\subseteq [T]$, we say that an adversary includes at most $M$ mistakes on $\Ucal$ if all weak learners $i\in[L]$ have misclassification error at most $M$ on times in $\Ucal$, i.e., for all rounds $t\geq 1$, $|\set{t'\in[t]\cap \Ucal:y_{i,t'} \notin\{y_{t'},\perp\} }|\leq M$.

\end{definition}

Following arguments inspired from the informal overview above, we obtain the following guarantee for the \textsc{Delete} algorithms.

\begin{lemma}\label{lemma:guarantee_delete_algo}
    Fix $C\geq 0$ and $M\geq 1$. Consider an adversary for $L$ weak learners which includes at most $C$ predictions per round and at most $M$ mistakes on some (potentially adaptive) set of times $\Ucal$. Denote by $\hat y_t(s)\in\{0,1,\perp\}$ the value selected by $\textsc{Delete}_{s,\frac{C}{2}}$ at iteration $t$. Then, for any $s_{\max} \geq 1$ and defined iteration $n\geq 1$, there exists $s\in\{0,\ldots,s_{\max}\}$ such that
    \begin{equation*}
        \abs{ \set{t\in[n]\cap\Ucal:  \hat y_t(s) \notin \{y_t,\perp\} }}  \leq \frac{8M \ceil{\log L}}{s_{\max}+1}|\Ucal|.
    \end{equation*}
\end{lemma}

In words, \cref{lemma:guarantee_delete_algo} shows that there is always a choice of the deletion number $s\in\{0,\ldots,s_{\max}\}$ that yields good misclassification error for $\textsc{Delete}_{s,C/2}$. We can therefore aggregate these algorithms using a standard learning-with-experts to essentially achieve the best misclassification error among these. For simplicity, we can abstain whenever any of these $\textsc{Delete}_{s,C/2}$ abstains, and otherwise we use the Weighted Majority Algorithm (WMA), which conveniently is deterministic. We include the full description of the resulting algorithm $\textsc{Aggregate}_{s_{\max},C/2}$ in \cref{alg:aggregate_1_layer}. Importantly, all considered algorithms $\textsc{Delete}_{s,C/2}$ delete at most $s_{\max}$ predictions for each weak learner, which will be crucial to bound the overall abstention error. Specifically, up to the $s_{\max}$ predictions deleted for each weak learner, these algorithms always make a prediction when at least $C/2$ weak learners made a prediction.

\begin{algorithm}[t]

\caption{Aggregate algorithm for one layer $\textsc{Aggregate}_{s_{\max},C}$}\label{alg:aggregate_1_layer}
\LinesNumbered
\everypar={\nl}

\KwIn{number of experts $L$, minimum number of predictions $C$, maximum deletion $s_{\max}\geq 0$}

\vspace{3mm}

Initialize weights $w_s=1$ for all $s\in\{0,\ldots,s_{\max}\}$

\For{$t\geq 1$}{
    Let $\hat y_t(s)$ denote the selection of $\textsc{Delete}_{s,C}$ for $s\in\{0,\ldots,s_{\max}\}$

    \lIf{$\exists s\leq s_{\max}: \hat y_t(s)=\perp$}{Abstain $\hat y_t=\perp$} % and keep $w_{s,t+1}=w_{s,t}$ for $s\in\{0,\ldots,s_{\max}\}$}
    \Else{
    Select $\hat y_t = \argmax_{y\in\{0,1\}} \sum_{s\leq s_{\max}: \hat y_t(s)=y} w_s$ %w_{s,t}$

    For $s\in\{0,\ldots,s_{\max}\}$, if $\hat y_t(s)\neq y_t$, update 
    $w_s\gets w_s/2$
    %$w_{s,t+1}=\frac{w_{s,t}}{2}$; otherwise keep $w_{s,t+1}=w_{s,t}$
    }
}
\end{algorithm}

Using standard misclassification guarantees for WMA \cite{littlestone1994weighted}, we obtain the following bound on the misclassification error of $\textsc{Aggregate}_{s_{\max},C}$.

\begin{lemma}\label{lemma:guarantee_aggregate_algo}
    Fix $C,s_{\max}\geq 0$ and $M\geq 1$. Denote by $\Ucal$ the set of times when $\textsc{Delete}_{s_{\max},C/2}$ makes a prediction $\hat y_t(s_{\max})\in\{0,1\}$, i.e. $\Ucal=\set{t:\abs{\set{i\in[L]: y_{i,t}\neq \perp \text{ and } \abs{\set{t'<t:y_{i,t'}\neq \perp}} \geq s_{\max} }} \geq C/2}$. Note that $\textsc{Aggregate}_{s_{\max},C/2}$ exactly makes predictions on times in $\Ucal$.
    
    Consider an adversary for $L$ weak learners which includes at most $C$ predictions per round and at most $M$ mistakes on $\Ucal$. Denote by $\hat y_t$ the selected value of $\textsc{Aggregate}_{s_{\max},C/2}$ at iteration $t$. Letting $n$ be the final iteration, we have
    \begin{equation*}
         \abs{ \set{t\in[n]:  \hat y_t \notin \{y_t,\perp\} }} \leq \frac{24M \ceil{\log L}}{s_{\max}+1}|\Ucal| + 3\log (s_{\max}+1).
    \end{equation*}
\end{lemma}

\begin{remark}
    Note that the guarantee \cref{lemma:guarantee_aggregate_algo} holds whenever weak learner make few mistakes at times when \textsc{Aggregate} makes a prediction---even if weak learners make many mistakes outside of predictions times for \textsc{Aggregate}. This subtlety will not be important for the present abstention model where we receive the true label $y_t$ at each iteration $t$, but will be useful when we consider the censored abstention model in which we only receive labels at times of prediction.
\end{remark}

We are now ready to construct our full algorithm, which uses the subroutine \textsc{Aggregate} on multiple layers. To ensure that each weak learner $i\in[L]$ makes at most $M$ mistakes, we first delete all predictions $y_{i,t}\in\{0,1\}$ from weak learners $i\in[L]$ which made $M$ misclassification mistakes in previous rounds, that is, we replace $y_{i,t}\gets \perp$. We first start with $\textsc{Aggregate}_{s_{\max},L/2}$, forming the first layer of our final algorithm. From \cref{lemma:guarantee_aggregate_algo}, this achieves low misclassification error among all times when the number of weak learners making a prediction lied within $[L/2,L]$, after removing the first $s_{\max}$ predictions of all weak learners. On the other hand, it may abstain when less than $L/2$ weak learners make predictions. For these remaining times, we then use $\textsc{Aggregate}_{s_{\max},L/4}$, forming the second layer. This ensures low misclassification error whenever the number of weak learners making a prediction lied within $[L/4,L/2]$, after removal of $s_{\max}$ predictions per weak learner in this layer. We continue building process until the $\ceil{\log L}$-th layer, which takes care of all remaining times. The final algorithm $\textsc{Boosting}_{s_{\max}}$ is summarized in \cref{alg:main_exp_version}. Importantly, the subroutine $\textsc{Aggregate}_{s_{\max},2^{-j}L}$ is run on different subset of times for each layer $j$, which corresponds to when the number of weak learner predictions (after potential modifications) is at most $2^{1-j}L$. Within \cref{alg:main_exp_version}, we denote by $\Qcal_j$ the set of times on which this layer-$j$ subroutine is run, and write $\textsc{Aggregate}_{s_{\max},2^{-j}L}^{\Qcal_j}$ to emphasize this point.

\begin{algorithm}[t]

\caption{Final boosting algorithm $\textsc{Boosting}_{s_{\max}}$}\label{alg:main_exp_version}
\LinesNumbered
\everypar={\nl}

\KwIn{number of experts $L$, maximum per-layer deletion $s_{\max}$, mistake tolerance $M$}

\vspace{3mm}

For $j\in[\ceil{\log L}]$ initialize times $\Qcal_j\gets \emptyset$ and deletion counts $s_{i,j}\gets s_{\max}$ for all $i\in[L]$

\For{$t\geq 1$}{
    Observe recommendation $y_{i,t}$ for each expert $i\in[L]$

    For $i\in[L]$, if $|\{s<t: y_{i,s}\notin\{y_s,\perp\}\}|\geq M$, delete prediction: $z_{i,t}\gets \perp$; otherwise $z_{i,t}\gets y_{i,t}$

    Define $n_t\gets |\{i\in[L]: z_{i,t}\neq \perp\}|$
    
    \While{$n_t>0$}{
        Let $j$ such that $n_t\in(2^{-j}L,2^{1-j}L]$

        Let $\hat y_t^{(j)}$ be the value selected by $\textsc{Aggregate}_{s_{\max},2^{-j}L}^{\Qcal_j}$ for a new iteration with weak learner recommendations $(z_{i,t})_{i\in[L]}$. Add $\Qcal_j\gets \Qcal_j\cup\{t\}$

        \For{$i\in[L]$ with $z_{i,t}\neq \perp$ and $s_{i,j}>0$}{
            Delete prediction: $z_{i,t}\gets \perp$ and update $s_{i,j}\gets s_{i,j}-1$
        }

        \lIf{$\hat y_t^{(j)}\neq \perp$}{
            Predict $\hat y_t=\hat y_t^{(j)}$ and \textbf{break}
        }
        Update $n_t \gets |\{i\in[L]: z_{i,t}\neq \perp\}|$
    }

    \lIf{$n_t=0$}{Abstain $\hat y_t=\perp$}

}
\end{algorithm}
Combining together the learning guarantees for all \textsc{Aggregrate} subroutines from \cref{lemma:guarantee_aggregate_algo}, we obtain the desired guarantee for the \textsc{Boosting} algorithm.

\paragraph{Proof sketch for \cref{thm:combining_weak_lerners}.} First, to check that the algorithm is well-defined we can verify that the while loop on each iteration $t\geq 1$ terminates. Indeed, on each loop, the $\textsc{Aggregate}_{s_{\max},2^{-j}L}$ makes a prediction $\hat y_t^{(j)}\in\{0,1\}$ if at least $2^{-j}L$ learner predictions remain after deleting $s_{\max}$ predictions per weak learner (see line 11 of \cref{alg:main_exp_version}). Roughly speaking, at each loop at least half of the predictions are deleted, ensuring that $j$ strictly increases on each loop and hence terminates after $\ceil{\log L}$ rounds when no weak learner predictions remain.

Next, since the boosting procedure uses the predictions of the $\textsc{Aggregate}_{s_{\max},2^{-j}L}$ algorithm when there were at most $2^{1-j}L$ weak learner predictions (see definition of $j$ on line 7 of \cref{alg:main_exp_version}), \cref{lemma:guarantee_aggregate_algo} bounds the misclassification error for each layer $j\in[\ceil{\log L}]$. The final misclassification is essentially the sum of misclassification error for all layers, incurring an extra factor $\log L$.

For the abstention error, note that by construction each layer $j\in[\ceil{\log L}]$ deleted at most $s_{\max}$ predictions for each weak learner (see the initialization of deletion counts in line 1 of \cref{alg:main_exp_version}). In summary, throughout the procedure only deleted $\leq s_{\max} \ceil{\log L}$ predictions for weak learner with at most $M$ misclassification error. This gives the desired abstention error bound since the \textsc{Boosting} algorithm makes a prediction whenever at least one weak learner prediction remains (see line 15 of \cref{alg:main_exp_version}).\hfill $\blacksquare$

\section{Censored sequential learning with abstentions}
\label{sec:censored_model}

In this section, we turn our attention to the \emph{censored} model in which the only difference with the learning protocol described in \cref{sec:framework} is that in step 4, the learner only observes the label $y_t=f^\star(x_t)$ if it made a prediction $\hat y_t\in\{0,1\}$. This naturally models settings in which we only receive feedback if an action was taken. In this variant model, we show that all positive results described in \cref{sec:main results} for \algoname still hold, up to minor modifications of the algorithm, and slightly worse learning guarantees dependency in the VC dimension $d$.

The main difficulty in this censored setting is that the labels used to run both weak learners and the boosting procedure are not necessary available. To resolve this issue we make the following modifications to the algorithms.

\paragraph{Extending weak learners.} Note that weak learners require labels $y_t$ to update the function class $\Fcal_{t+1}\gets \Fcal_t^{x_t\to y_t}$, see lines 9-10 of \cref{alg:weak_learner}. However, recall that updates are only required infrequently---$\Ocal(d\log 1/\epsilon)$ times for the $\texttt{update}=\texttt{restricted}$ update policy, see \cref{subsec:overview_algorithm} for intuitions. Up to slightly increasing the number of weak learners, we may therefore add these labels to the weak learner parameters. For convenience, we also include times of updates. Altogether, weak learners $\textsc{WL}(\Tcal,\Ucal,z,u)$ now have four parameters: the subset of times $\Tcal$ used for shattering probability estimation, corresponding labels $z=(z_t)_{t\in\Tcal}$, a set of function-class updates times $\Ucal\subset [T]$ and corresponding labels $u=(u_t)_{t\in\Ucal}$. Formally, we replace the update rules in lines 9-10 of \cref{alg:weak_learner} by
\begin{quote}
    If $t\in\Ucal$, update $\Fcal_{t+1}=\Fcal_t^{x_t\to u_t}$; otherwise keep $\Fcal_{t+1}=\Fcal_t$.
\end{quote}

\begin{corollary}
\label{cor:iid_examples_censored}
There exists a universal constant $c_0>0$ such that the following holds. Fix $\Fcal:\Xcal\to\{0,1\}$ of VC dimension $d$, $\epsilon\in(0,1)$, and an adversary. Let $N= \ceil{2000d^2 /\epsilon}$.
Let $\Tcal$ denote the set of first $mN$ non-corrupted rounds. Define $z:=(y_t)_{t\in\Tcal}$. Let $\Ucal$ denote the (random) set of misclassifcation times for $\textsc{WL}(\Tcal,z)$ with $\texttt{update}=\texttt{restricted}$, and define $u=(y_t)_{t\in\Ucal}$.

If (1) the adversary is oblivious and $m \geq 80 d^2\log(1/\epsilon)\log T$, or (2) the adversary is adaptive and $m \geq c_0(D\log(D)+8D+3\log(d/\epsilon))$ where $D$ is the $\ceil{5d^2\log \frac{1}{\epsilon}}$-reduction dimension of $\Fcal$, then
\begin{equation*}
    \textsc{MisErr}(\Tcal,\Ucal,z,u) < 5d^2\log\frac{1}{\epsilon} \quad \text{and} \quad \Ebb[\textsc{AbsErr}(\Tcal,\Ucal,z,u)] \leq 18\epsilon T,
\end{equation*}
where $\textsc{MisErr}(\Tcal,\Ucal,z,u)$ and $\textsc{AbsErr}(\Tcal,\Ucal,z,u)$ respectively denote the misclassification and abstention error of the weak learner $\textsc{WL}(\Tcal,\Ucal,z,u)$ as defined in \cref{alg:weak_learner}.
\end{corollary}

\paragraph{Adapting the boosting strategy.} The only place where the $\textsc{Boosting}$ algorithm uses labels is in line 4 of \cref{alg:main_exp_version}, where we delete predictions of weak learners which made at least $M$ mistakes in prior rounds. This cannot be implemented in the censored model since we do not have access to values $y_t$ when the final algorithm abstained $\hat y_t=\perp$. Naturally, we can relax this procedure by deleting predictions of weak learners for which we have \emph{observed} at least $M$ mistakes in prior rounds when the true value $y_t$ was observed. Formally, we replace line 4 of \cref{alg:main_exp_version} by
\begin{quote}
    For $i\in[L]$, if $|\{s<t: \hat y_t\neq \perp\text{ and }y_{i,s}\notin\{y_s,\perp\}\}|\geq M$, delete prediction: $z_{i,t}\gets \perp$; otherwise $z_{i,t}\gets y_{i,t}$.
\end{quote}
We denote by \textsc{C-Boosting} the corresponding boosting strategy for the censored setting.

This modification is sufficient to ensure that after this round of deletions, at each layer $j\in[\ceil{\log L}]$, each weak learner makes at most $M$ incorrect predictions at times when the corresponding $\textsc{Aggregate}_{s_{\max},2^{-j}L}$ algorithm makes predictions. While this does not ensure few mistakes for each weak learner, this is sufficient to apply the learning guarantee on \textsc{Aggregate} from \cref{lemma:guarantee_aggregate_algo}. Hence, the same proof immediately yields the following variant of \cref{thm:combining_weak_lerners}. 

\begin{corollary}
\label{cor:combining_weak_lerners_censored}
    Fix parameters $M\geq 0$ and $s_{\max}\in [T]$. Then, for $L$ experts, $\textsc{C-Boosting}_{s_{\max},M}$ achieves the following guarantee against any (adaptive) adversary: 
    \begin{align*}
        \textsc{MisErr}\leq c_0\frac{(M+\log T)T\log L}{s_{\max}} \quad\text{and}\quad   
        \textsc{AbsErr} \leq s_{\max}\ceil{\log L} + \min_{i\in[L]:M_i<M} A_i,
    \end{align*}
    % for some universal constant $c_0>0$,
    % where $M_i$ and $A_i$ respectively denote the misclassification and abstention error of expert $i\in [L]$.
    where $A_i$ denotes the abstention error of expert $i\in[L]$, and 
    $$M_i\coloneqq \abs{\set{t\in [T]:y_{i,t}\notin\{y_t, \perp\},\hat{y}_t\neq\perp}}$$ 
    denotes the misclassification error of expert $i$ during times in $\set{t\in [T]:\hat{y}_t\neq\perp}$ when $\textsc{C-Boosting}_{s_{\max},M}$ makes a prediction.
\end{corollary}

Putting these two results together yields the desired guarantee for the censored model from \cref{thm:censored_main}. Specifically, we consider the \textsc{C-AbstainBoost} algorithm which runs \textsc{C-Boosting} over all weak learners of the form $\text{WL}(\Tcal,\Ucal,z,u)$ where $\Tcal,\Ucal\subseteq[T]$ have size bounded by the parameters $|\Tcal|\leq mN$ and $|\Ucal|\leq M$. For completeness, we detail the algorithm in \cref{subsec:censored_algo_description}.

\section{Conclusion}
\label{sec:conclusion}

In the sequential learning with abstention framework of \cite{goel2023adversarial}, we showed that achieving sublinear abstention and misclassification errors is possible for general VC classes without knowing $\mu$, for oblivious adversaries. We also identified structural properties of the function class $\Fcal$---finite so-called \emph{reduction dimension}---enabling learning against adaptive adversaries, which may be of independent interest for other abstention learning models. Together with corresponding lower bounds, these results show the existence of a polynomial-form tradeoff between abstention and misclassification errors. This work naturally leaves open the following two directions.

\paragraph{Tight characterizations of tradeoffs between abstention and misclassification errors.} At the high-level, our positive results for oblivious adversaries show that $\Ocal(T^{3\alpha})$ misclassification can be achieved while ensuring $\tilde \Ocal(T^{1-\alpha})$ abstention error, leaving a gap compared to the corresponding $\Omega(T^\alpha)$ misclassification error lower bound. The lower bound turns out to be tight for VC-1 classes for instances, but the situation is unclear for general VC classes. We note, however, that with further parameter information (e.g., which deletion parameter $s\leq s_{\max}$ to use within \textsc{Delete} procedures), slight modifications of the algorithm can yield stronger misclassification guarantees, opening potential opportunities for further improvements.

\paragraph{Extending results to adaptive adversaries for general VC classes.} Our guarantees for adaptive adversaries are limited to function classes with finite reduction dimension (and finite VC dimension). It remains open whether one can remove such constraints and achieve successful learning for all VC classes, which leads to the following open question:

\begin{quote}
    \textit{Can we achieve $\text{poly}(d)T^{1-\Omega(1)}$ misclassification and abstention error for any function class with VC dimension $d$, against adaptive adversaries?}
\end{quote}

As a remark, the main limitation in our current approach is in achieving accurate estimations of the shattering probability from i.i.d.\ samples. Specifically, our approach reduces the problem to the case when the learner has a priori access to a few i.i.d.\ samples---say $T^{1-\Omega(1)}$---from the clean distribution $\mu$. If learning can be efficiently performed with few misclassifications in this setting, then our boosting procedure ensures learning guarantees for the general case in which the learner has no prior distributional information.

In concurrent work, \cite{edelman2026reliableabstention} used different complexity measures than the VC dimension---the inference dimension sometimes used in active learning and their introduced certificate dimension---to achieve non-trivial error bounds. Specifically, their algorithm achieves $\Ocal(T^{1-1/k})$ misclassification and abstention error for function classes with inference dimension $k$. In light of our universal error bound $\tilde \Ocal(d^2 T^{3/4})$ for all function classes with VC-dimension $d$ against oblivious adversaries\footnote{using \cref{thm:main_oblivious_upper_bound} for $\alpha=1/4$}, it is natural, however, to expect that the VC dimension also characterizes learnability for the adaptive case.

\acks{The authors are grateful to Abhishek Shetty for useful discussions.}

\bibliographystyle{alpha}
\bibliography{refs}

\newpage
\appendix

\crefalias{section}{appendix} % uncomment if you are using cleveref
\section{Concentration inequalities}

\begin{theorem}[Median estimator, e.g. Theorem 2 from \cite{lugosi2019mean}]
\label{thm:median}
    Fix $\delta\in(0,1)$ and let $X_1,\ldots,X_n$ be i.i.d.\ samples from a distribution with mean $\mu$ and variance $\sigma^2$, with $n\geq \ceil{8\log(1/\delta)}$. Then, with probability at least $1-\delta$,
    \begin{equation*}
        |\text{Median}(X_1,\ldots,X_n)-\mu| \leq 2\sigma.
    \end{equation*}
\end{theorem}

\section{Proofs of \cref{sec:main results}}

\subsection{\textsc{C-AbstainBoost} description}
\label{subsec:censored_algo_description}

The algorithm in the censored model is analogous: we perform the \textsc{Boosting} strategy over all weak learners $\text{WL}(\Tcal,\Ucal,z,u)$ for subsets $\Tcal\subseteq[T]$ and labelings $(z_t)_{t\in\Tcal}$ of bounded size $mN$, and any subset $\Ucal\subseteq[T]$ and labelings $(u_t)_{t\in\Ucal}$ of size at most $M$. The algorithm is given in \cref{alg:censored_abstain_boost}. 

\begin{algorithm}[t]

\caption{\textsc{C-AbstainBoost}}\label{alg:censored_abstain_boost}

%\noindent\rule{\linewidth}{0.4pt}
\LinesNumbered
\everypar={\nl}
\SetAlgoNoEnd
\setcounter{AlgoLine}{0}

\KwIn{horizon $T$, precision $\epsilon\in[0,1]$, number of subsets $m$, subset size $N$, maximum per-layer deletion $s_{\max}\in[T]$, mistake tolerance $M$}

\vspace{3pt}

Run $\textsc{C-Boosting}$ with maximum per-layer deletion $s_{\max}$, mistake tolerance $M$, on all $L:=\sum_{k\leq mN}2^k\binom{T}{k}\cdot \sum_{k\leq M}2^k\binom{T}{k}$ weak learners $\textsc{WL}(\Tcal,\Ucal,z,u)$ run with precision $\epsilon$ and number of subsets $m$, for $\Tcal\subseteq[T]$ and $z=(z_t)_{t\in\Tcal}\in\{0,1\}^\Tcal$ with $|\Tcal|\leq mN$, and $\Ucal\subseteq[T]$ and $(u_t)_{t\in\Ucal}\in\{0,1\}^\Ucal$ with $|\Ucal|\leq M$

\end{algorithm}

\subsection{Parameter choice for our main upper bound results}

We first prove our guarantee for \algoname against oblivious adversaries.

\vspace{3pt}
\begin{proof}[of Theorem \ref{thm:main_oblivious_upper_bound}]
This proof combines \cref{thm:iid_examples_v2} and \cref{thm:combining_weak_lerners}. We use the following parameters (unless mentioned otherwise): $\texttt{update}=\texttt{always}$, $\epsilon=d^2 \log^{5/3}(T) T^{-\alpha}$, $m= \ceil {8\log(dT/\epsilon)}$, $N= \ceil{2000d^2 / \epsilon}$, $s_{\max}=\ceil{d^2 \log^{4/3}(T) T^{1-2\alpha}}$ and $M=5d^2 \log(1/\epsilon)$. The number of weak learners is 
$L=\sum_{k\leq mN}2^k\binom{T}{k}$. Provided that $s_{\max}\leq T$,  \cref{thm:combining_weak_lerners} implies
\begin{align*}
    \textsc{MisErr}\lesssim  M\log L \cdot\frac{T}{s_{\max}}  \lesssim d^2 \log^{5/3} T\cdot \frac{  T^{2\alpha}}{\epsilon} \asymp T^{3\alpha}.
\end{align*}
And by \cref{thm:iid_examples_v2}, for each run, the weak learner $\textsc{WL}(\Tcal,z)$ with $(\Tcal,z)$ matching the first $mN$ times of uncorrupted samples and labels satisfies
\[\textsc{MisErr}(\Tcal,z)< M\text{ and }\Ebb\left[\textsc{AbsErr}(\Tcal,z)\right]\leq 18\epsilon T.\]
Hence by \cref{thm:combining_weak_lerners}, provided that $\epsilon<1$,
\begin{align*}
    \Ebb [\textsc{AbsErr}] &\leq s_{\max}\ceil{\log L} + \Ebb \left[\min_{\substack{i\in[L]\\
    \textsc{MisErr}(\Tcal_i,z_i)< M}} \textsc{AbsErr}(\Tcal_i,z_i)\right]\\
    &\leq s_{\max}\ceil{\log L} + \Ebb\left[\textsc{AbsErr}(\Tcal,z)\right]\\
    &\lesssim d^4 T^{1-2\alpha}\log^{10/3} (T)/\epsilon+\epsilon T \asymp d^2\log^{5/3} (T) \cdot T^{1-\alpha}.
\end{align*}
We now consider the edge cases when $s_{\max}> T$ or $\epsilon\geq 1$, both of which imply $d^2\log^{5/3}(T) >T^{\alpha}$. In particular, the desired abstention error statement is vacuous. In particular, $0$ misclassification error can be achieved by never making any prediction, achieved by taking $M=0$.
\end{proof}

Next, we prove our guarantee for \algoname against adaptive adversaries, assuming the initial function class $\Fcal$ has a bounded reduction dimension.

\vspace{3pt}
\begin{proof}[of Theorem \ref{thm:main_adaptive_upper_bound}]
The proof is exactly the same as for \cref{thm:main_oblivious_upper_bound} except with the following different parameters: we choose $\texttt{update}=\texttt{restricted}$, $\epsilon=d^2 (D\log D+\log T)^{2/3} \log(T) \cdot T^{-\alpha}$, $m = \ceil{c_0(D\log(D)+8D+3\log(d/\epsilon))}$, $N= \ceil{2000d^2 / \epsilon}$, $M=5d^2\log(1/\epsilon)$, and $s_{\max}=\ceil{d^2(D\log D+\log T)^{1/3} \log(T) \cdot  T^{1-2\alpha}}$. First, if $\epsilon\leq 1/T$ then the desired misclassification error is vacuous: we could have achieved 0 abstention error by setting $s_{\max}=0$ and $M=\infty$. We suppose this is not the case from now, in which case, $D$ bounds the $\ceil{5d^2\log(1/\epsilon)}$-reduction dimension of $\Fcal$. Again, we have $L=\sum_{k\leq mN}2^k\binom{T}{k}$ weak learners, provided that $s_{\max}\leq T$ and $\epsilon<1$. Applying \cref{thm:combining_weak_lerners} again yields $\textsc{MisErr}\lesssim T^{3\alpha}$. Next, the same arguments using \cref{thm:iid_examples_v2,thm:combining_weak_lerners} show that
    \begin{align*}
        \Ebb [\textsc{AbsErr}] &\lesssim s_{\max}\ceil{\log L} + \epsilon T \asymp d^2 (D\log D+\log T)^{2/3} \log(T) \cdot T^{1-\alpha}.
\end{align*}
    We next turn to the case $s_{\max}>T$ or $\epsilon\geq 1$. Again, in both cases, the abstention bound is vacuous and we can safely choose $M=0$ to get the desired guarantee.
\end{proof}

Last, we prove the learning guarantees in the censored model.

\vspace{3pt}

\begin{proof}[of Theorem \ref{thm:censored_main}]
We first check that \textsc{C-AbstainBoost} can indeed be run in the censored model. First, by construction, weak learners do not require any true value to be run, hence it suffices to focus on the \textsc{C-Boosting} procedure. Note that it only requires knowing the true value $y_t$ at times when it makes a prediction (see replaced line 4 in \cref{alg:main_exp_version}), which also corresponds to times when the final procedure \textsc{C-AbstainBoost} makes a prediction.

The proofs of the abstention and misclassification error bounds for adaptive adversaries are identical to those in the proof of \cref{thm:main_adaptive_upper_bound} with the same choice of parameters.
We now focus on the abstention and misclassification errors for oblivious adversaries.
This combines \cref{cor:iid_examples_censored} and \cref{cor:combining_weak_lerners_censored} using the analogous arguments as in the proof of \cref{thm:main_oblivious_upper_bound} and the following parameters for \textsc{C-AbstainBoost}: $\epsilon=d^{10/3} \log^{7/3}(T) T^{-\alpha}$, $m= \ceil {80d^2\log(1/\epsilon)\log(T)}$, $N= \ceil{2000d^2 / \epsilon}$, $s_{\max}=\ceil{d^{8/3} \log^{5/3}(T) T^{1-2\alpha}}$ and $M=5d^2 \log(1/\epsilon)$ (unless the desired abstention error is vacuous in which case we can achieve $0$ misclassification error by taking $M=0$ as in the proof of \cref{thm:main_oblivious_upper_bound}. Note that the number of weak learners is now 
$L=\sum_{k\leq Nm}2^k\binom{T}{k}\cdot \sum_{k\leq M}2^k\binom{T}{k}\leq (2T)^{Nm+M}=T^{O(Nm)}$.
\end{proof}

\subsection{Proof of our main lower bound result}

Apart from deriving error upper bounds for distribution-free abstention learning, we also explored the lower bounds of this problem. We show that there is a trade-off between misclassification and abstention errors. Our result corresponds to region 2 in \cref{fig:tradeoff results}.% and \cref{fig:VC_1_tradeoffs}. 

\vspace{3pt}

\begin{proof}[of \cref{thm:lowerbound}]
    Fix a parameter $N=4T$.
    We start by constructing a function class $\Fcal$ using the following tree structure: let $\Tcal:=\bigcup_{t=0}^T [N]^t$ be a rooted full $N$-ary tree of depth $T$, that is, $x\in\Tcal$ is an ancestor of $y$ if and only if $x$ is a prefix of $y$ (incuding $y=x$) and denote $x\preceq y$ accordingly. For convenience, we denote by $\Tcal_i:=\bigcup_{t=0}^i [N]^t$ the subtree of depth $i$ for $i\in[T]$. We then consider the function class $\Fcal:\Tcal\to\{0,1\}$ containing all indicator paths from the root to a node of the tree as follows:
    \begin{equation*}
        \Fcal=\set{\1[\cdot\preceq x_0]: x_0\in\Tcal}.
    \end{equation*}
    One can easily check that $\Fcal$ has VC dimension one due to its tree structure---the tree structure known to be a characterization of VC-1 classes \citep{ben20152}.

    We now fix a learning algorithm $alg$ o $\Fcal$ such that there is a constant $A\in[1/2,T/2]$ such that for any oblivious adversary on $\Fcal$, $\Ebb[\textsc{AbstentionError}]\leq A$. Next, let $i_{\max}:=\floor{T/(2A)}$. In the following, we recursively construct for each layer $i\leq i_{\max}$ a sequence $(x_t)_{t\leq k_i}$ within $\Tcal_i$ together with $x^\star_i\in[N]^i$ such that $k_i\leq 2A\cdot i$ and run on the realizable sequence $(x_t,y_t=\1[x_t\preceq  x^\star_i])_{t\leq k_i}$, the algorithm $alg$ makes the following misclassification errors:
    \begin{equation*}
        \forall j\leq i,\quad \Pbb[\hat y_{k_j} = 1-y_{k_j}] \geq \frac{1}{8},
    \end{equation*}
    where $\hat y_t$ denotes the prediction of $alg$ at time $t$. We let $k_0=0$ and $x_0^\star$ be the root so that this trivially holds for $i=0$.

    Fix $i\in[i_{\max}]$ and suppose that the construction is complete for layer $i-1$. Let $\mu_i$ be the uniform distribution over the $N$ children of $x_{i-1}^\star$. As a note, $\mu_i$ is a distribution over $[N]^i$. We now complete the sequence $(x_t)_{t\leq k_{i-1}}$ using i.i.d.\ samples from $\mu_i$. Denote by $(\tilde x_t)_{t\in[T]}$ the corresponding sequence as well as the responses $\tilde y_t=\1[\tilde x_t\preceq x_{i-1}^\star]$. Note that the data $(\tilde x_t,\tilde y_t)_{t\in[T]}$ is realizable and can be viewed as containing $k_{i-1}$ corrupted samples followed by $T-k_{i-1}$ non-corrupted samples from $\mu_i$. Hence, when run on this realizable sequence, we must have
    \begin{equation*}
        \sum_{t=k_{i-1}}^T \Pbb[\hat y_t=\perp] \leq A.
    \end{equation*}
    In particular, there exists $k_i>k_{i-1}$ with $k_i\leq k_{i-1}+2A$ such that $\Pbb[\hat y_{k_i}=\perp]\leq 1/2$. Additionally, since the samples $x_t$ for $t\in(k_{i-1},k_i]$ are i.i.d.\ sampled from $\mu_i$ which is uniform over $N$ elements, we have
    \begin{equation*}
        \Pbb[\exists t\in(k_{i-1},k_i), \tilde x_{k_i}= \tilde x_t] \leq \sum_{t=k_{i-1}+1}^{k_i-1} \Pbb[\tilde x_{k_i} = \tilde x_t] \leq \frac{T}{N} =\frac{1}{4}.
    \end{equation*}
    As a result, introducing the $\Ecal_i:=\{\forall t\in(k_{i-1},k_i), \tilde x_{k_i}\neq \tilde x_t\}$, we have
    \begin{equation*}
        \Pbb[\{\hat y_{k_i} \neq \perp\} \cap \Ecal_i ] \geq \frac{1}{2}-\frac{1}{4}=\frac{1}{4}.
    \end{equation*}
    By the law of total probability, we can therefore fix elements $x_t\in[N]^i$ for $t\in(k_{i-1},k_i]$ such that (1) $x_{k_i}\notin \{x_t,t\in(k_{i-1},k_i)\}$ and (2) $alg$ run on the data $(x_t,\1[x_t\preceq x_{i-1}^\star])_{t<k_i}$ and tested on $x_{k_i}$ satisfies $\Pbb[\hat y_{k_i}\neq \perp] \geq 1/4$. This completes the construction of $(x_t)_{t\leq k_i}$.
    Next, if $\Pbb[\hat y_{k_i}=0]\geq 1/8$ we pose $x_i^\star:=x_{k_i}$. Otherwise, we define $x_i^\star$ to be any child of $x_{i-1}^\star$ which does not belong to $\{x_t,t\in(k_{i-1},k_i]\}$. Note that this is possible since $x_{i-1}^\star$ has $N>T$ children. We then pose $y_{k_i}=\1[x_{k_i}\preceq x_i^\star]$. By construction, when $alg$ is run on $(x_t,y_t)_{t\leq k_i}$, we obtained
    \begin{equation}\label{eq:mistake_proba}
        \Pbb[\hat y_{k_i} = 1-y_{k_i}] \geq \frac{1}{8}.
    \end{equation}
    It only remains to check that for all $t\in[k_i]$ one has
    \begin{equation}\label{eq:induction_hypothesis}
        y_t=\1[x_t\preceq x_i^\star].
    \end{equation}
    By definition, this holds for $t=k_i$.
    Since $x_{i-1}^\star$ is the parent of $x_i^\star$, the functions $\1[\cdot\preceq x_{i-1}^\star]$ and $\1[\cdot \preceq x_i^\star]$ coincide on $\Tcal_i$ and hence, \cref{eq:induction_hypothesis} directly holds for all $t\leq k_{i-1}$. Finally, for any $t\in(k_{i-1},k_i)$, by construction we have $x_t\neq x_i^\star$ and as a result, $y_t=\1[x_t\preceq x_{i-1}^\star]=0=\1[x_t\preceq x_i^\star]$ (recall that $x_t\neq x_i^\star$ both have depth $i$ while $x_{i-1}^\star$ has depth $i-1$). 
    
    This ends the induction and the construction of the realizable sequence $(x_t,y_t)_{t\leq k_{i_{\max}}}$. We can complete it to define $(x_t,y_t=\1[x_t\preceq x_{i_{\max}}^\star])_{t\leq T}$ arbitrarily. Note that this is valid since $k_{i_{\max}}\leq 2A i_{\max}\leq T$. By construction, on this data, $alg$ has the following misclassification error:
    \begin{equation*}
        \Ebb[\textsc{MisErr}] \geq \Ebb\sqb{\sum_{t=1}^T \1[\hat y_t=1-y_t]} \geq \sum_{i=1}^{i_{\max}} \Pbb[\hat y_{k_i}\neq y_{k_i}] =\frac{i_{\max}}{8} \geq  \frac{1}{32} \frac{T}{A},
    \end{equation*}
    where in the last inequality we used the fact that $A\leq T/2$.
    This ends the proof of the desired misclassification error bound. Note that this construction can be embedded for all $T\geq 1$ within the function class defined by a full $\Nbb$-ary tree of infinite depth (which still has VC dimension 1). This ends the proof.
\end{proof}

\section{Reduction-dimension bounds for various hypothesis classes}
\label{sec:reduction_dimension_examples}

In this section, we argue that many classical and practical VC function classes have small reduction dimension, making it a reasonable complexity measure for many hypothesis classes (in fact, we are not aware of ``natural'' VC function classes that have infinite reduction dimensions). Below, we give some examples of reduction-dimension computations for linear classifiers, VC-1 classes, axis-aligned rectangles, and subsets of bounded size.

\paragraph{Linear classifiers.}
We start with linear separators in dimension $d$:
\begin{equation*}
    \Fcal_{\text{lin}}^d:=\set{x\in\Rbb^d\mapsto \1[a^\top x\geq b] : a\in\Rbb^d, b\in\Rbb},
\end{equation*}
which have VC dimension $d+1$. In the following, we show that linear separators have bounded reduction dimension. The proof uses oriented matroid arguments to translate the combinatorial counting problem into an algebraic problem. We then employ Warren's theorem \cite[Theorem 2]{warren1968lower}, a classical result in algebraic geometry which bounds the number of topological components defined by polynomial equations.

\begin{lemma}\label{lemma:reduction_dim_linear_sep}
    Let $d\geq 1$. Then, for any $l\geq 1$, $\Fcal_{\mathrm{lin}}^d$ has $l$-reduction dimension at most $c_1 d^2 l$ for some universal constant $c_1>0$.
\end{lemma}

\begin{proof}
    We refer to \cite{bjorner1999oriented} for a detailed exposition of oriented matroids. We will only use their standard properties.
    For any vector configuration $X=(x_1,\ldots,x_n)$ in $\Rbb^{d+1}$ we consider the associated oriented matroid $M_X:= ([n],\Lcal_X)$ where $\Lcal_X$ is the collection of all covectors of $X$ generated by linear functionals:
    \begin{equation*}
        \Lcal_X = \set{(\mathrm{sgn}(y^\top x_i))_{i\in[n]} : y\in\Rbb^d}.
    \end{equation*}
    Here, $\mathrm{sgn}:\Rbb\to \{+,-,0\}$ is the sign function. Since the vectors lie in $\Rbb^{d+1}$, the rank $r$ of the matroid $M_X$ is at most $d+1$ (the rank of a matroid is the unique length of maximal chains in $\Lcal_X$, for the partial order given by the inclusion of supports, minus one).
    Additionally, the chirotope associated with the vector configuration $X$ is given by
    \begin{equation*}
        \chi_X:(i_1,\ldots,i_r)\in[n]^r \mapsto \mathrm{sgn}(\det(x_{i_1},\ldots,x_{i_r})) \in \{+,-,0\}.
    \end{equation*}
    Importantly, the matroid $M_X$ is characterized by its chirotope $\chi_X$ (by default with the determinant we fixed a global sign orientation), e.g. see \cite[Theorem 3.5.5 or Corollary 3.5.12]{bjorner1999oriented}. Crucially, this implies that given collection of all signs of determinants $\mathrm{sgn}(\det(x_{i_1},\ldots,x_{i_{d+1}})) \in\{+,-,0\}$ between any $d+1$ vectors of the configuration, we can recover the collection of their covectors $\Lcal_X$. 

    We are now ready to bound the $l$-reduction dimension of linear separators. For convenience, for any point $z\in\Rbb^d$ we will denote by $\tilde z = (z^\top ,1)^\top \in\Rbb^{d+1}$ its homogeneisation. Note that the projection of linear separators onto a set of points can be obtained from the collection of covectors of its homogeneized vector configuration (the covectors store the information $\{+,-,0\}$ while linear separators either merge $+$ and $0$, or $-$ and $0$). Further, for set of $l'\leq l$ datapoints $A=\{(a_i,y_i) :i\in[l']\}$, given the projection of linear separators onto $\{x_1,\ldots,x_n,a_1,\ldots,a_{l'}\}$ we can recover $\Fcal_{\mathrm{lin}}\cap A|_{\{x_1,\ldots,x_n\}}$, simply by focusing on projections whose value on $a_i$ equals $y_i$ for all $i\in[l']$. Altogether, this shows that
    \begin{align*}
        &\abs{\left\{\Fcal_{\mathrm{lin}}^d \cap A|_{\{ x_1,\ldots,x_n\}} :A\subseteq\Xcal\times\{0,1\} ,|A|\leq l \right\}} \leq \sum_{l'=0}^l\abs{\set{M_{(\tilde x_1,\ldots,\tilde x_n,\tilde a_1,\ldots,\tilde a_{l'})} : a_1,\ldots,a_{l'}\in\Rbb^d}}\\
        &= \sum_{l'=0}^l \left|\left\{ (i_1,\ldots,i_{d+1})\in[n+l']^{d+1}\mapsto \mathrm{sgn}(\det(z_{i_1},\ldots,z_{i_{d+1}})) \right.\right.\\
        &\qquad\qquad\qquad\qquad\qquad\qquad\qquad\qquad \left.\left.: Z=(\tilde x_1,\ldots,\tilde x_n,\tilde a_1,\ldots,\tilde a_{l'}), \, a_1,\ldots,a_{l'}\in\Rbb^d \right\}\right| .
    \end{align*}
    For a fixed $l'\in[l]$, denote by $N_{l'}$ the $l'$-th term within the sum. Note that for any indices $i_1,\ldots,i_{d+1}\in [n+l']$, the function $\det(z_{i_1},\ldots,z_{i_{d+1}})$ is a polynomial in the variable $(a_1,\ldots,a_{l'})\in (\Rbb^d)^{l'}$ of degree at most $d+1$.
    Therefore, $N_{l'}$ exactly corresponds to the number of sign patterns for all equations of the type $\det(z_{i_1},\ldots,z_{i_{d+1}})=0$ in $(\Rbb^d)^{l'}$. We can then use a variant of Warren's theorem \cite[Proposition 5.5]{alon1995tools} which bounds this number by $(8e M K / D)^D$, where $M$ is the number of polynomials over $D$ variables and $K$ is their maximum degree, if $2M\geq D$. In our context, this gives
    \begin{equation*}
        N_{l'} \leq \paren{8e \frac{d+1}{dl'} \binom{n+l'}{d+1}}^{dl'}.
    \end{equation*}
    Plugging this into the previous bound shows that
    \begin{equation*}
        \abs{\Fcal_{\mathrm{lin}}^d \cap A|_{\{ x_1,\ldots,x_n\}} :A\subseteq\Xcal\times\{0,1\} ,|A|\leq l } \leq (n+l)^{c_1 d^2 l},
    \end{equation*}
    for some constant $c_1>0$. In turn, this shows that $D_l \leq c_2 d^2 l$ for some universal constant $c_2>0$.
\end{proof}

\paragraph{Function classes with VC dimension $1$.}
Next, we check that VC-1 classes have bounded reduction dimension, using their convenient tree representation \cite{ben20152}.

\begin{proposition}\label{prop:reduction_dim_VC_1}
    Let $\Fcal:\Xcal\to\{0,1\}$ be a function class with VC dimension $1$. Then, for any $l\geq 1$, $\Fcal$ has $l$-reduction dimension $l+\Ocal(1)$.
\end{proposition}
\begin{proof}
Fix a function class $\Fcal$ of VC dimension 1. From \cite[Theorem 4]{ben20152} without loss of generality, we may consider that we are given a tree ordering $\preceq$ on $\Xcal$ (that is, a partial order on $\Xcal$ such that every initial segment $I_x=\{y\in\Xcal:y\preceq x\}$ is linearly ordered) such that all functions $f\in \Fcal$ is an initial segment with respect to $\preceq$, i.e., for any $x\preceq y\in\Xcal$, if $f(y)=1$ then $f(x)=1$.

Consider any dataset $A\subseteq\Xcal\times\{0,1\}$. Let $B$ be the dataset obtained by deleting datapoints $(x,0)\in A$ such that (1) there exists $(z,0)\in A$ with $z\prec x$ or (2) there exists $(z,1)\in A$ with $z\nprec x$; as well as deleting datapoints $(x,1)\in A$ such that there exists $(z,1)\in A$ with $x\prec z$. Note that $\Fcal\cap A=\Fcal\cap B$. 

Fix test points $S=\{x_1,\ldots,x_n\}\in\Xcal$. We focus on the tree ordering $\preceq$ restricted to ancestors of $S$, that is, $T:=\{x\in\Xcal: x\preceq y, y\in S\}$.
Then, note that when restricted to $S$, the function class $\Fcal\cap B|_S$ is equivalent to replacing each datapoint $(x,y)\in B$ with any $(\tilde x,y)$ where $\tilde x$ has the same $\preceq$ comparisons in terms of $\preceq$ with all datapoints in $S\cup \{z:(z,y)\in B\}$. We recall that by construction, all $(x,0)\neq(z,0)\in B$ are such that $x$ and $z$ are incomparable. Further, there is at most one datapoint $(x,1)\in B$ with label $1$ and in that case, all other datapoints $(z,0)$ in $B$ are such that $x\prec z$. Hence, we may choose one such representative for each possible $\preceq$-comparisons with $S$ for all datapoints in $B$ with label $0$. Note that the number of such possible $\preceq$-comparisons is bounded by the number of possible nodes in a rooted tree with leaves within $S$: it is at most $2|S|=2n$.
For the potential label $1$ we additionally need to respect the fact that $x\prec z$ for all $(z,0)\in B$: up to duplicating representatives this gives $4n$ possible choices. In summary,
\begin{equation*}
    |\{\Fcal\cap A|_S:A\subseteq\Xcal\times\{0,1\},|A|\leq l\}| \leq 1+\sum_{r\leq l}((2n)^r+(2n)^{r-1}\cdot 4n)\lesssim n^l,
\end{equation*}
where the first term $1$ counts the empty function class if the dataset is not realizable, the second (resp. last) term corresponds to datasets $B$ containing only $0$ labels (resp. at least one $1$ label). This ends the proof.
\end{proof}

\paragraph{Axis-aligned rectangles.}
We now turn to axis-aligned rectangles in $\Rbb^d$ defined as follows:
\begin{equation*}
    \Rcal^d:=\left\{B_{a,b}:x\in\Rbb^d\mapsto \1[a\leq x \leq b]: a,b\in\Rbb^d\right\},
\end{equation*}
also have finite reduction dimension. Here, inequalities between vectors are meant component-wise. That is, for $a,b\in\Rbb^d$, $a\leq b$ when $a_i\leq b_i$ for all $i\in[d]$. We recall that axis-aligned rectangles have VC dimension $2d$.

\begin{proposition}\label{prop:axis_aligned}
    For any $l\geq 1$, $\Rcal^d$ has $l$-reduction dimension at most $4d+1$.
\end{proposition}

\begin{proof}
    Note that for any dataset $A\subseteq\Rbb^d\times\{0,1\}$, the reduced function class $\Rcal^d\cap A$ can be reduced to separate constraints on each coordinate $i\in[d]$ on the range allowed for the interval $[a_i,b_i]$ for functions $f_{a,b}\in \Rcal\cap A$:
    \begin{equation*}
        \Rcal\cap A = \{f_{a,b}: \forall i\in[d], a_i\in(z_i^{(1)}, z_i^{(2)}] ,b_i\in[z_i^{(3)} ,z_i^{(4)}) \},
    \end{equation*}
    where $z_i^{(1)}=\max\{z_i:(z,0)\in A, \exists (\tilde z,1)\in A, z_i<\tilde z_i\}$, $z_i^{(2)}=\min\{z_i:(z,1)\in A, \exists (\tilde z,0)\in A, \tilde z_i<z_i\}$, and similarly for $z_i^{(3)},z_i^{(4)}$, with the convention $\max\emptyset=-\infty$ and $\min\emptyset=+\infty$. 

    Hence, given $n$ test points $S=\{x^1,\ldots,x^n\}\subseteq\Rbb^d$, the projection $\Rcal\cap A|_S$ only depends on the relative ordering of $z_i^{(1)} < z_i^{(2)}\leq z_i^{(3)} <z_i^{(4)}$ compared to $\{x_i^1,\ldots,x_i^n\}$, for each $i\in[n]$. This gives at most $(n+1)^4$ choices for each $i\in[n]$. In summary,
    \begin{equation*}
        |\{\Rcal^d\cap A|_S:A\subseteq\Xcal\times\{0,1\},|A|\leq l\}| \leq 1+(n+1)^{4d},
    \end{equation*}
    where the additional $1$ comes from the case when $A$ is not realizable.
    Hence, $\Rcal^d$ has $l$-reduction at most $4d+1$.
\end{proof}

\paragraph{Subsets of bounded size.}
Another classical function class example is that of subsets of size at most $d$: for any instance space $\Xcal$, we define
\begin{equation*}
    \Scal^d_\Xcal = \{x\in\Xcal\mapsto\1[x\in S]: S\subseteq\Xcal,|S|\leq d\},
\end{equation*}
which has VC dimension $d$ by construction, provided that $|\Xcal|\geq d$. We can easily check that these have low reduction dimension.

\begin{proposition}\label{prop:subsets_d_reduction_dim}
    For any $\Xcal$ and $l\geq 1$, $\Scal_\Xcal^d$ has $l$-reduction dimension $l+\Ocal(\log d)$.
\end{proposition}

\begin{proof}
    The main observation is that for any dataset $A$, $\Scal_\Xcal^d\cap A$ fixes the value of the function on $A_x:=\{x:(x,y)\in A\}$ and on $\Xcal\setminus S$, it exactly corresponds to the function class $\Scal_{\Xcal\setminus A_x}^{d-r}$ where $r=|\{(x,1)\in A\}|$ is the number of ones in the dataset (since $r$ ones have been fixed, there are at most $d-r$ remaining). Hence, projected on test points $S=\{x_1,\ldots,x_n\}$, the possible $l$-reduced classes correspond to fixing the value on at most $l$ points in $S\cap A_x$, and on the rest of the test points, the reduced class is characterized by a single number $d-r\leq d-p$ where $p=|\{(x,1):x\in S\cap A_x\}|$ is the number of ones on the $l$ datapoints coinciding with test points in $S$. Hence,
    \begin{equation*}
        |\{\Scal_X^d\cap A|_S:A\subseteq\Xcal\times\{0,1\},|A|\leq l\}| \leq 1+\binom{n}{l}\sum_{p=0}^d \binom{l}{p} (d-p+1)\leq 1+edn^l
    \end{equation*}
    where the additional $1$ corresponds to the case when $A$ is not realizable. This ends the proof.
\end{proof}

\section{Proofs of \cref{sec:weak learners}}
We start by proving \cref{lemma:calculation_variance}, which controls the variance of estimation.

\vspace{3pt}
\begin{proof}[of \cref{lemma:calculation_variance}]
Recall that
\[
\hat{\rho}_{k}^{S}(\Fcal):=\frac{1}{\binom{N}{k}}\sum_{S'\subseteq S,|S'|=k}1[S'\text{ is shattered by }\Fcal].
\]
First consider 
\[
\begin{aligned} & \mathbb{P}\left(S_{1},S_{2}\text{ both shattered by }\Fcal\right)\\
 & =\mathbb{P}\left(S_{1}\text{ is shattered by }\Fcal\left|S_{2}\text{ is shattered by }\Fcal\right.\right)\cdot\mathbb{P}\left(S_{2}\text{ is shattered by }\Fcal\right)\\
 & \leq\mathbb{P}\left(S_{1}\backslash S_{2}\text{ is shattered by }\Fcal\right)\cdot\mathbb{P}\left(S_{2}\text{ is shattered by }\Fcal\right)
\end{aligned}
\]
where the last inequality uses the independence of ``$S_{1}\backslash S_{2}\text{ is shattered by }\Fcal$''
and ``$S_{2}\text{ is shattered by }\Fcal$''. Recall that
$\rho_{r}(\Fcal,\mu)$ denotes the probability of a set with
cardinality $r$ being shattered, and since we are fixing $\Fcal$
and $\mu$ here we can write $\rho_{r}$ as a short cut for $\rho_{r}(\Fcal,\mu)$.
Given $N\ge2k$, we have
\[
\begin{aligned}\mathbb{E}\left[\hat{\rho}_{k}^{S}(\Fcal)^{2}\right] & =\binom{N}{k}^{-2}
\sum_{r=0}^{k}
\sum_{\substack{S_1,S_2\subseteq S\\
|S_1|=|S_2|=k\\
|S_1\cap S_2|=r}}
\Pbb\bigl(S_1,S_2 \text{ are both shattered by }\Fcal\bigr).\\
& =\binom{N}{k}^{-2}
\sum_{r=0}^{k}
\sum_{\substack{S_1,S_2\subseteq S\\
|S_1|=|S_2|=k\\
|S_1\cap S_2|=r}}
\rho_{k}\rho_{k-r}\\
 & =\binom{N}{k}^{-2}\sum_{r=0}^{k}\binom{N}{r}\binom{N-r}{k-r}\binom{N-k}{k-r}\rho_{k}\rho_{k-r}\\
 & =\rho_{k}\left(\sum_{r=0}^{k}\frac{\binom{k}{r}\binom{N-k}{k-r}}{\binom{N}{k}}\rho_{k-r}\right).
\end{aligned}
\]
Therefore 
\begin{equation}\label{eq:variance_general_bound}
\text{Var}\left[\hat{\rho}_{k}^{S}(\Fcal)\right]\leq\rho_{k}\left(\sum_{r=1}^{k}\frac{\binom{k}{r}\binom{N-k}{k-r}}{\binom{N}{k}}\left(\rho_{k-r}-\rho_{k}\right)\right).
\end{equation}
Consider a hypergeometric distribution $X\sim\text{Hypergeometric}(N,k,k)$
then
\[
\text{Var}\left[\hat{\rho}_{k}^{S}(\Fcal)\right]\leq\rho_{k}\mathbb{E}\left[\rho_{k-X}\1(X\geq1)\right].
\]
For simplicity, write $\eta:=k^2 /N$. By assumption, $\rho_{l}\leq c\eta^{l-(k-1)}$ for
all $l\leq k-1$. Therefore,
\begin{equation}
\text{Var}\left[\hat{\rho}_{k}^{S}(\Fcal)\right]\leq\rho_{k}\mathbb{E}\left[c\eta^{1-X}\right].\label{eq:Var_rho_1}
\end{equation}
Let $R\sim\text{Binom}(k,\frac{k}{N})$. A standard comparison inequality
for sampling with/without replacement \cite[Theorem 4]{hoeffding1963} implies
that $\mathbb{E}[\eta^{-X}]\leq\mathbb{E}[\eta^{-R}]$. And
hence 
\begin{equation}
\mathbb{E}[\eta^{-X}]\le(1-\frac{k}{N}+\frac{k}{N\eta})^{k}\leq\exp\left(\Big(\tfrac{1}{\eta}-1\Big)\frac{k^{2}}{N}\right).\label{eq:Var_rho_2}
\end{equation}
Combining \cref{eq:Var_rho_1} and \cref{eq:Var_rho_2} we have
\[
\text{Var}\left[\hat{\rho}_{k}^{S}(\Fcal)\right]\leq\rho_{k}c\eta\cdot\exp\left(\Big(\tfrac{1}{\eta}-1\Big)\frac{k^{2}}{N}\right)<3c\eta\rho_k.
\]
This ends the proof.
\end{proof}

The following lemma guarantees accurate estimation of shattering probabilities up to a $(1\pm 0.2)$ factor, which can be obtained as an immediate result of \cref{lemma:calculation_variance} under the assumption that estimation error $\left|\rho_{k}^{\mathcal{S}}(\Fcal)-\rho_{k}(\Fcal,\mu)\right|$ is bounded by standard deviation $\sigma_k ^N (\Fcal,\mu)$. Nevertheless, it's useful to state it separately as it will be repeatedly used in the proof of \cref{thm:iid_examples_v2}.
\begin{lemma}\label{lemma:inductive_concentration_rho_k}
Let $\Fcal:\mathcal{X}\to\{0,1\}$ and $\mu$ a distribution
on $\mathcal{X}$ and let $\mathcal{S}=\left\{ S_{i}\right\} _{i\in[m]}$
where each $S_{i}$ contains $N$ iid samples from $\mu$. Let $2000 d^{2}/N\leq\epsilon<1$. 
Assume that there is an integer $k\in[d]$ and constant $c'$ such
that for $0\leq l\leq k-1$,
\[
\rho_{l}(\Fcal,\mu)\leq1.2c'\epsilon^{l-(k-1)}.
\]
Then, given
\begin{equation}
\left|\rho_k ^\mathcal{S}(\Fcal)-\rho_{k}(\Fcal,\mu)\right|\leq2\sigma_k^N(\Fcal,\mu)\label{eq:median_assumption}
\end{equation}
we have
\[
\left|\rho_{k}^{\mathcal{S}}(\Fcal)-\rho_{k}(\mathcal{F},\mu)\right|<0.1\sqrt{c'\epsilon\rho_{k}(\Fcal,\mu)}
\]
which implies that, whenever $\rho_{k}^{\mathcal{S}}(\Fcal)\geq c'\epsilon$,
we have $\rho_{k}(\Fcal,\mu)\in\rho_{k}^{\mathcal{S}}(\Fcal)\cdot[0.8,1.2]$
as well.
\end{lemma}
\begin{proof}
Since $k^2/N\leq\epsilon$, we have $\rho_{l}\leq1.2c'(k^2/N)^{l-(k-1)}$ for $0\leq l\leq k-1 $, by
\cref{lemma:calculation_variance},
\[
\text{Var}\left[\hat{\rho}_{k}^{S_{1}}(\mathcal{F})\right]<3.6(k^2/N) c' \rho_{k}(\mathcal{F},\mu).
\]
With condition (\ref{eq:median_assumption}), we have
\[
\begin{aligned}\left|\rho_{k}^{\mathcal{S}}(\mathcal{F})-\rho_{k}(\mathcal{F},\mu)\right| & \leq2\sqrt{3.6 (k^2/N) c' \rho_{k}(\mathcal{F},\mu)}\\
 & \leq2\sqrt{3.6\epsilon c' \rho_{k}(\mathcal{F},\mu)/2000}\\
 & <0.1\sqrt{c'\epsilon\rho_{k}(\mathcal{F},\mu)}.
\end{aligned}
\]
Therefore, whenever $\rho_{k}^{\mathcal{S}}(\mathcal{F})\geq c'\epsilon$,
we have
\[
\left|\rho_{k}^{\mathcal{S}}(\mathcal{F})-\rho_{k}(\mathcal{F},\mu)\right|<0.1\sqrt{\rho_{k}^{\mathcal{S}}(\mathcal{F})\rho_{k}(\mathcal{F},\mu)},
\]
which implies that $\rho_{k}(\mathcal{F},\mu)\in\rho_{k}^{\mathcal{S}}(\mathcal{F})\cdot[0.8,1.2]$
as well.
\end{proof}

In the adaptive case, the condition \cref{eq:median_assumption} does not naturally hold---as discussed in \cref{sec:weak learners}, the function class $\Fcal$ can be badly correlated with $\Scal$. We next prove \cref{lemma:uniform_convergence_median}, which resolves this issue by showing universal concentration over all possible
function classes of the form $\Fcal\cap A$ where $\Fcal$ is the initial hypothesis class and $A$ can be any subset
of less than $l$ data points in $\Xcal\times\{0,1\}$.

\vspace{3pt}

\begin{proof}[of \cref{lemma:uniform_convergence_median}]
    The first step of the proof involves lifting the problem from $\Xcal$ to $\Xcal_s:= \set{S\subseteq \Xcal, |S|=s}$ then applying standard uniform concentration bounds on a carefully chosen VC class $\Gcal:\Xcal_s\to \{0,1\}$. First, we denote by $\mu_s$ the distribution of $\{x_i,i\in[s]\}$ where $(x_i)_{i\in[s]}\overset{iid}{\sim}\mu$ are i.i.d.\ samples from $\mu$. We then define the function class $\Gcal:\Xcal_s\to\{0,1\}$ as the collection of all functions of the form
    \begin{equation*}
        \phi_{A,I}:S\in \Xcal_s \mapsto \1_I[\hat \rho_k^S(\Fcal\cap A) ],
    \end{equation*}
    for any subset $A\subseteq\Xcal\times\{0,1\}$ with $|A|\leq l$, and any interval $I\subseteq \Rbb_+$.
    
    We now bound the VC dimension of this function class $\Gcal$. Fix any test sets $S_1,\ldots, S_m\in\Xcal_s$ for $m\geq D_l/s$. Note that given values $z_1,\ldots,z_m\in\Rbb_+$, the set of possible projections of functions $\1_I[\cdot]$ for closed intervals $I\subseteq\Rbb_+$ onto $\{z_1,\ldots,z_m\}$ is upper bounded by $\binom{m+1}{2}+1$---we can order these elements by increasing order then decide of the start and end point of those which belong to the interval $I$. In particular,
    \begin{align*}
        \abs{\Gcal|_{\{S_1,\ldots,S_m\}}} &\leq \abs{\set{(\hat\rho_k^{S_i} (\Fcal\cap A))_{i\in[m]}: A\subseteq\Xcal\times\{0,1\},|A|\leq l}}\cdot m(m+1)\\
        &\overset{(i)}{\leq} \abs{\set{\Fcal\cap A|_{S_1\cup\ldots\cup S_m}: A\subseteq\Xcal\times\{0,1\},|A|\leq l}}\cdot m(m+1)\\
        &\overset{(ii)}{\leq} (ms)^{D_l} \cdot m(m+1).
    \end{align*}
    In $(i)$ we noted that given the projection of the function class $\Fcal\cap A$ onto $S_1\cup\ldots\cup S_m$, we can construct all estimates of the form $\hat\rho_k^{S_1}(\Fcal\cap A),\ldots,\hat\rho_k^{S_m}(\Fcal\cap A) $. Indeed, $\hat\rho_k^S(\Fcal)$ only evaluates the function class $\Fcal$ on points in $S$. In $(ii)$ we used the definition of the restriction dimension together with $ms\geq D_l$. Recall that if $\{S_1,\ldots,S_m\}$ are shattered by $\Gcal$, then $\abs{\Gcal|_{\{S_1,\ldots,S_m\}}}=2^m$. Together with the previous bound, this shows that $\Gcal$ has VC dimension at most $c_0D_l\log (sD_l)$ for some $c_0\geq 1$.

    We can then apply standard VC uniform concentration bounds on $\Gcal$, e.g. \cite[Theorem 12.5]{devroye2013probabilistic}, which shows that there is a universal constant $c_1$ such that the following holds for any $m\geq c_1 (D_l\log (sD_l)+\log \frac{1}{\delta})$. Let $S_1,\ldots,S_m\overset{iid}{\sim}\mu_s$, with probability at least $1-\delta$, for any subset $A\subseteq\Xcal\times\{0,1\}$ with $|A|\leq l$ and interval $I\subseteq\Rbb_+$,
    \begin{equation}\label{eq:uniform_concentration_bound}
        \abs{\frac{1}{m}\abs{\set{i\in[m]:\hat\rho_k^{S_i}(\Fcal\cap A)\in I }} - \Pbb_{S\sim\mu_s}[\hat\rho_k^{S_i}(\Fcal\cap A)\in I]} \leq \frac{1}{8}.
    \end{equation}
    Now note that for any $A\subseteq\Xcal\times\{0,1\}$, by Chebyshev's inequality we have
    \begin{align*}
        \Pbb_{S\sim \mu_s}\paren{|\hat\rho_k^S(\Fcal\cap A) - \rho_k(\Fcal\cap A,\mu)| \leq 2\sigma_k^s(\Fcal\cap A,\mu)} \leq \frac{1}{4}.
    \end{align*}
    In particular, for the interval $I_A := \rho_k(\Fcal\cap A,\mu) + 2\sigma_k^s(\Fcal\cap A,\mu) \cdot [-1,1]$, this precisely shows $\Pbb_{S\sim\mu_s}[\hat\rho_k^{S_i}(\Fcal\cap A)\in I_A] \geq \frac{3}{4}$. Together with \cref{eq:uniform_concentration_bound}, we obtained with probability at least $1-\delta$, for any set $A$ of at most $l$ datapoints,
    \begin{equation*}
         \abs{\set{i\in[m]:|\hat\rho_k^{S_i}(\Fcal\cap A) - \rho_k(\Fcal\cap A,\mu)| \leq 2\sigma_k^s(\Fcal\cap A,\mu) }} > \frac{m}{2}.
    \end{equation*}
    In turn, this implies that the median value of $\hat\rho_k^{S_i}(\Fcal\cap A)$ for $i\in[m]$ belongs to the desired interval:
    \begin{equation*}
        |\rho_k^\Scal(\Fcal\cap A) - \rho_k(\Fcal\cap A,\mu)| \leq 2\sigma_k^s(\Fcal\cap A,\mu).
    \end{equation*}
    This ends the proof.
\end{proof}

We are now ready to prove the main guarantee for the weak learner $\textsc{WL}(\Tcal,z)$ with appropriate parameters.
\vspace{3pt}
\begin{proof}[of \cref{thm:iid_examples_v2}]
Before analyzing the miclassification and abstention error, we make a few remarks.
First, note that for the considered subset $\Tcal$, the samples $(x_t)_{t\in\Tcal}$ are i.i.d.\ sampled from $\mu$. Further, by construction, since these are the first non-corrupted times and the vector $z$ contains their correct labels, in the interval of time $[\max\Tcal]$ the learner $\textsc{WL}(\Tcal,z)$ makes no classification nor abstention mistakes. As a result, it suffices to focus on the period of time $(\max\Tcal,T]$. 

For each $t\in\left\{ \max\mathcal{T}+1,\ldots,T\right\} $, we define the set of function classes encountered by the learner via
\[
\mathcal{G}\coloneqq\set{ \mathcal{F}_{t}, \mathcal{F}_{t}^{x_t\to 0}, \mathcal{F}_{t}^{x_t\to 1}: t\in(\max\mathcal{T},T] }.
\]
Note that these are random function classes.
Next, we define the event
\begin{equation}\label{eq:event_Ecal}
\mathcal{E}\coloneqq\left\{ \forall\tilde \Fcal\in\mathcal{G},k\in[d],\quad\left|\rho_{k}^{\Scal}(\tilde \Fcal)-\rho_{k}(\tilde \Fcal,\mu)\right|\leq 2\sigma_k^N(\tilde\Fcal)\right\}.
\end{equation}
For now, let us assume $\Pbb(\Ecal^c)\leq 3\epsilon$. We will verify this assumption later for both oblivious and adaptive adversaries, after analyzing the misclassification and abstention error.

\paragraph{Misclassification error.}
First, we analyze the misclassification error. 
By construction of weak learners, misclassification errors only occur at times $t\in[T]$ when $\rho_{1}^{\mathcal{S}}(\mathcal{F}_{t})>\epsilon$
and $\min_{y\in\left\{ 0,1\right\} }\rho_{k_t}^{\mathcal{S}}(\mathcal{F}_{t}^{x_{t}\to y})<0.9\rho_{k_t}^{\mathcal{S}}(\mathcal{F}_{t})$, where $k_t$ is defined as in line 5 of \cref{alg:weak_learner}.
Thus, since we predict with the label corresponding to $\arg\max_{y\in\left\{ 0,1\right\} }\rho_{k_t}^{\mathcal{S}}(\mathcal{F}_{t}^{x_{t}\to y})$
and a mistake occurred, we have 
\[
\rho_{k_t}^{\mathcal{S}}(\mathcal{F}_{t+1})=\rho_{k_t}^{\mathcal{S}}(\mathcal{F}_{t}^{x_{t}\to y_{t}})<0.9\rho_{k_t}^{\mathcal{S}}(\mathcal{F}_{t}).
\]
Meanwhile, notice that combining $\rho_{1}^{\mathcal{S}}(\mathcal{F}_{t})>\epsilon$
and $\forall l\in[k_t],\rho_{l}^{S}(\mathcal{F}_{t})>\epsilon\cdot\rho_{l-1}^{S}(\mathcal{F}_{t})$,
we have
\[
\rho_{k_t}^{\mathcal{S}}(\mathcal{F}_{t})>\epsilon\cdot\epsilon^{k_t-1}=\epsilon^{k_t}.
\]
For each fixed $k\in[d]$, $\rho_k(\Fcal_t)$ is non-increasing since the $\Fcal_t$ are non-increasing throughout the learning procedure. In particular, $\rho_{k}^{\mathcal{S}}(\Fcal_t)$ can be decreased
by a factor $9/10$ at most $k\log_{\frac{10}{9}}\left(\frac{1}{\epsilon}\right)$
times. Hence 
\[
\textsc{MisErr}(\mathcal{T},z)\leq\sum_{k=1}^{d}k\log_{\frac{10}{9}}\left(\frac{1}{\epsilon}\right)< 5d^{2}\log\left(\frac{1}{\epsilon}\right).
\]
Note that this result on misclassification error holds \emph{regardless of} event $\Ecal$. 

\paragraph{Abstention error.}
Consider a time $t\in[\max\mathcal{T}+1,T]$ for which $k_t$ was defined, that is $\rho_1^S(\Fcal_t)>\epsilon$. We start by proving by induction that under event $\Ecal$,
\begin{equation}
\rho_{k'}(\mathcal{F}_{t},\mu)\in\rho_{k'}^{\mathcal{S}}(\mathcal{F}_{t})\cdot[0.8,1.2],\quad\forall 0\leq k'\leq k_t.\label{eq:3-claim1}
\end{equation}
Note that the base case $k'=0$ is true because $\rho_{0}(\mathcal{F}_{t},\mu)=\rho_{0}^{\mathcal{S}}(\mathcal{F}_{t})=1$. Suppose that for some $k'\in[k_t]$ this holds for all $0\leq l\leq k'-1$.
By construction of the weak learner, we have
\begin{equation*}
    \rho_{l}^{\mathcal{S}}(\mathcal{F}_{t})\leq\rho_{k'-1}^{\mathcal{S}}(\mathcal{F}_{t})\epsilon^{l-(k'-1)},\quad 0\leq l\leq k'-1.
\end{equation*}
Together with the induction hypothesis, this implies
\begin{equation}
\rho_{l}(\mathcal{F}_{t},\mu)\leq1.2\rho_{l}^{\mathcal{S}}(\mathcal{F}_{t})\leq1.2\rho_{k'-1}^{\mathcal{S}}(\mathcal{F}_{t})\epsilon^{l-(k'-1)} , \quad  0\leq l\leq k'-1.\label{eq:3-3}
\end{equation}
Moreover, the event $\mathcal{E}$ implies that 
\begin{equation}
\left|\rho_{k'}^{\Scal}(\Fcal_t)-\rho_{k'}(\mathcal{F}_{t},\mu)\right|\leq2 \sigma_{k'}^N(\Fcal_t).\label{eq:3-2}
\end{equation}
Also, recall that by construction, we have $\rho_{k'}^{\mathcal{S}}(\mathcal{F}_{t})>\epsilon\rho_{k'-1}^{\mathcal{S}}(\mathcal{F}_{t})$. Note that this also holds when $k'=1$ since by construction $\rho_{1}^{\mathcal{S}}(\mathcal{F}_{t}) \geq \epsilon$ and $\rho_{0}^{\mathcal{S}}(\mathcal{F}_{t})=1$. Altogether, we can now apply \cref{lemma:inductive_concentration_rho_k} with the parameter $c'=\rho_{k'-1}^{\mathcal{S}}(\mathcal{F}_{t})$, which gives the desired induction
\[
\rho_{k'}(\mathcal{F}_{t},\mu)\in\rho_{k'}^{\mathcal{S}}(\mathcal{F}_{t})\cdot[0.8,1.2].
\]
This ends the proof of \cref{eq:3-claim1}.

Using the same arguments, we can check that \cref{eq:3-2,eq:3-3} still hold for $k'=k_t +1$. Hence, applying Lemma \ref{lemma:inductive_concentration_rho_k}
gives
\[
\begin{aligned}\rho_{k_t+1}(\mathcal{F}_{t},\mu) & <\rho_{k_t+1}^{\mathcal{S}}(\mathcal{F}_{t})+0.1\sqrt{\epsilon\rho_{k_t}^{\mathcal{S}}(\mathcal{F}_{t})\rho_{k_t+1}(\mathcal{F}_{t},\mu)}\\
 & \overset{(i)}{\leq}\epsilon\rho_{k_t}^{\mathcal{S}}(\mathcal{F}_{t})+0.1\sqrt{\epsilon\rho_{k_t}^{\mathcal{S}}(\mathcal{F}_{t})\rho_{k_t+1}(\mathcal{F}_{t},\mu)},
\end{aligned}
\]
where in $(i)$ we used the definition of $k_t$ in line
5 of \cref{alg:weak_learner}: either $k_t=d$ in which case we already have $\rho_{k_t+1}^{\mathcal{S}}(\mathcal{F}_{t})=0$ or $k_t<d$ in which case we must have $\rho_{k_t+1}^{\mathcal{S}}(\mathcal{F}_{t}) \leq \epsilon \rho_{k_t}^{\mathcal{S}}(\mathcal{F}_{t})$. Consequently,
\begin{equation}\label{eq:3-1}
    \rho_{k_t+1}(\mathcal{F}_{t},\mu)\leq1.2\epsilon\rho_{k_t}^{\mathcal{S}}(\mathcal{F}_{t}).
\end{equation}

Next, we aim to show that for $y\in\left\{ 0,1\right\} $, given
$\rho_{k_t}^{\mathcal{S}}(\mathcal{F}_{t}^{x_{t}\to y})\geq0.9\rho_{k_t}^{\mathcal{S}}(\mathcal{F}_{t})$,
we have
\begin{equation}\label{eq:good_estimation_cases}
    \rho_{k_t}(\mathcal{F}_{t}^{x_{t}\to y},\mu)\in\rho_{k_t}^{\mathcal{S}}(\mathcal{F}_{t}^{x_{t}\to y})\cdot[0.8,1.2].
\end{equation}
To see this, first note that for any $l\in[0,k_t]$ we have
\[
\rho_{l}(\mathcal{F}_{t}^{x_{t}\to y},\mu)\leq\rho_{l}(\mathcal{F}_{t},\mu)\leq1.2\rho_{l}^{\mathcal{S}}(\mathcal{F}_{t})\leq1.2(\rho_{k_t}^{\mathcal{S}}(\mathcal{F}_{t})\slash\epsilon)\cdot\epsilon^{l-(k_t-1)}.
\]
Next, since the event $\mathcal{E}$ holds,
\[
\left|\rho_{k_t}^{\Scal}(\Fcal_t ^{x_t \to y})-\rho_{k_t}(\mathcal{F}_{t}^{x_{t}\to y},\mu)\right|\leq 2\sigma_{k_t}^N(\mathcal{F}_{t}^{x_{t}\to y}).
\]
By Lemma \ref{lemma:inductive_concentration_rho_k}
with $c'=\rho_{k_t}^{\mathcal{S}}(\mathcal{F}_{t})\slash\epsilon$,
\[
\begin{aligned}\left|\rho_{k_t}^{\mathcal{S}}(\mathcal{F}_{t}^{x_{t}\to y})-\rho_{k_t}(\mathcal{F}_{t}^{x_{t}\to y},\mu)\right| & <0.1\sqrt{\rho_{k_t}^{\mathcal{S}}(\mathcal{F}_{t})\rho_{k_t}(\mathcal{F}_{t}^{x_{t}\to y},\mu)}\\
 & \leq0.1\sqrt{\frac{10}{9}\rho_{k_t}^{\mathcal{S}}(\mathcal{F}_{t}^{x_{t}\to y})\rho_{k_t}(\mathcal{F}_{t}^{x_{t}\to y},\mu)}
\end{aligned}
\]
which implies $\rho_{k_t}(\mathcal{F}_{t}^{x_{t}\to y},\mu)\in\rho_{k_t}^{\mathcal{S}}(\mathcal{F}_{t}^{x_{t}\to y})\cdot[0.8,1.2]$, ending the proof of \cref{eq:good_estimation_cases}.

We are now ready to bound the abstention error. There are two possible cases for an abstention error: When $\rho_{1}^{\mathcal{S}}(\mathcal{F}_{t})\leq\epsilon$, since $x_{t}\notin D(\mathcal{F}_{t})$
implies $\hat{y}_{t}=y_{t}$, an error at time $t$ can only happen
when $x_{t}\sim\mu$ and $x_{t}\in D(\mathcal{F}_{t})$, as in line 12 of \cref{alg:weak_learner}. When $\rho_{1}^{\mathcal{S}}(\mathcal{F}_{t})>\epsilon$, an abstention error occurs if and only if the adversary has decided
not to corrupt at time $t$ (i.e. $c_{t}=0$) yet the learner still
abstains due to 
$\min_{y\in\{0,1\}}\rho_{k_t}^{\mathcal{S}}(\mathcal{F}_{t}^{x_{t}\to y})\geq0.9\rho_{k_t}^{\mathcal{S}}(\mathcal{F}_{t})$.
Let $\mathcal{H}_{t}$ denote the history before time $t$, hence $c_t\in\Hcal_t$.
Combining the two cases, the probability of an abstention error at time $t$ is bounded by
\begin{equation}\label{eq:Abs_error}
\begin{aligned}P_{t} & \coloneqq\mathbb{P}(x_{t}\in D(\mathcal{F}_{t}),\rho_{1}^{\mathcal{S}}(\mathcal{F}_{t})\leq\epsilon,c_{t}=0)\\
 & \qquad\qquad\qquad\qquad\qquad\qquad+\mathbb{P}(\min_{y\in\{0,1\}}\rho_{k_t}^{\mathcal{S}}(\mathcal{F}_{t}^{x_{t}\to y})\geq0.9\rho_{k_t}^{\mathcal{S}}(\mathcal{F}_{t}),\rho_{1}^{\mathcal{S}}(\mathcal{F}_{t})>\epsilon,c_{t}=0)\\
 & \leq\mathbb{P}(x_t\in D(\mathcal{F}_{t}),\rho_{1}^{\mathcal{S}}(\mathcal{F}_{t})\leq\epsilon,\mathcal{E}|c_t = 0)+\mathbb{P}(\min_{y\in\{0,1\}}\rho_{k_t}^{\mathcal{S}}(\mathcal{F}_{t}^{x_{t}\to y})\geq0.9\rho_{k_t}^{\mathcal{S}}(\mathcal{F}_{t}),\mathcal{E},c_{t}=0)+2\mathbb{P}(\mathcal{E}^{c})
\end{aligned}
\end{equation}
% Note that here $\mathbb{P}_{x\sim\mu}$ is taking expectation with respect to both $x$ and $\mathcal{H}_t$, with the distribution of $x\sim\mu$ being independent of $\mathcal{H}_t$. 
Applying Lemma \ref{lemma:inductive_concentration_rho_k} with $k=1$ and $c'=1$ yields that, under event $\mathcal{E}$,
$$\left|\rho_{1}^{\mathcal{S}}(\mathcal{F}_{t})-\rho_{1}(\mathcal{F}_{t},\mu)\right|<0.1\sqrt{\epsilon\rho_{1}(\mathcal{F}_{t},\mu)}.$$ Given $\rho_{1}^{\mathcal{S}}(\mathcal{F}_{t})\leq\epsilon$, we must have $\rho_{1}(\mathcal{F}_{t},\mu)\leq1.2\epsilon$. Otherwise, if $\rho_{1}(\mathcal{F}_{t},\mu)>1.2\epsilon$,$$\frac{5}{6}>\frac{\rho_{1}^{\mathcal{S}}(\mathcal{F}_{t})}{\rho_{1}(\mathcal{F}_{t},\mu)}>1-0.1\sqrt{\frac{\epsilon}{\rho_{1}(\mathcal{F}_{t},\mu)}}\geq1-\frac{0.1}{\sqrt{1.2}}\approx0.91,$$
which is clearly a contradiction. Therefore,
\begin{equation}\label{eq:0001}
\begin{aligned}\mathbb{P}(x_t\in D(\mathcal{F}_{t}),\rho_{1}^{\mathcal{S}}(\mathcal{F}_{t})\leq\epsilon,\mathcal{E}|c_t = 0) & \leq\mathbb{P}(x_t\in D(\mathcal{F}_{t}),\rho_{1}(\mathcal{F}_{t},\mu)\leq1.2\epsilon|c_t = 0)\\
 & \leq\mathbb{P}(x_t\in D(\mathcal{F}_{t})\left|\rho_{1}(\mathcal{F}_{t},\mu)\leq1.2\epsilon,c_t = 0\right.)\\
 & \leq1.2\epsilon.
\end{aligned}   
\end{equation}
where the last inequality is because the distribution of $x_t$ given $c_t=0$ is independent of $\Hcal_t$.

Now we turn to the second term in \cref{eq:Abs_error}. 
For $y\in\{0,1\}$, since $\rho_{k_t}^{\mathcal{S}}(\mathcal{F}_{t}^{x_{t}\to y})\geq0.9\rho_{k_t}^{\mathcal{S}}(\mathcal{F}_{t})$,
from the previous discussions we have that under $\mathcal{E}$,
\[
\rho_{k_t}(\mathcal{F}_{t}^{x_{t}\to y},\mu)\in\rho_{k_t}^{\mathcal{S}}(\mathcal{F}_{t}^{x_{t}\to y})\cdot[0.8,1.2],
\]
and by \cref{eq:3-claim1}, $\rho_{k_t}(\mathcal{F}_{t},\mu)\leq1.2\rho_{k_t}^{\mathcal{S}}(\mathcal{F}_{t})$.
Hence event $\left\{ \min_{y\in\{0,1\}}\rho_{k_t}^{\mathcal{S}}(\mathcal{F}_{t}^{x_{t}\to y})\geq0.9\rho_{k_t}^{\mathcal{S}}(\mathcal{F}_{t})\right\} \cap\mathcal{E}$
implies
\[
\min_{y\in\{0,1\}}\rho_{k_t}(\mathcal{F}_{t}^{x_{t}\to y},\mu)\geq0.8\min_{y\in\{0,1\}}\rho_{k_t}^{\mathcal{S}}(\mathcal{F}_{t}^{x_{t}\to y})\geq0.72\rho_{k_t}^{\mathcal{S}}(\mathcal{F}_{t})\geq0.6\rho_{k_t}(\mathcal{F}_{t},\mu).
\]
Consequently,
\begin{equation*}
\begin{aligned} & \mathbb{P}(\min_{y\in\{0,1\}}\rho_{k_t}^{\mathcal{S}}(\mathcal{F}_{t}^{x_{t}\to y})\geq0.9\rho_{k_t}^{\mathcal{S}}(\mathcal{F}_{t}),\mathcal{E},c_{t}=0)\\
 & \leq\mathbb{P}\left(\min_{y\in\{0,1\}}\rho_{k_t}(\mathcal{F}_{t}^{x_{t}\to y},\mu)\geq0.6\rho_{k_t}(\mathcal{F}_{t},\mu),c_{t}=0\right)\\
 & \leq\mathbb{P}_{x\sim\mu}\left(\min_{y\in\{0,1\}}\rho_{k_t}(\mathcal{F}_{t}^{x\to y},\mu)\geq0.6\rho_{k_t}(\mathcal{F}_{t},\mu)\right).
\end{aligned}\label{eq:5-1}
\end{equation*}
Note that here $\mathbb{P}_{x\sim\mu}$ is taking expectation with respect to both $x$ and $\mathcal{H}_t$, with the distribution of $x\sim\mu$ being independent of $\mathcal{H}_t$. 
Further, from \cite[Lemma 4.2]{goel2023adversarial}, or \cref{lemma:goel_main_lemma}, we have 
\begin{equation*}
\mathbb{P}_{x\sim\mu}\left(\min_{y\in\{0,1\}}\rho_{k_t}(\mathcal{F}_{t}^{x\to y},\mu)\geq0.6\rho_{k_t}(\mathcal{F}_{t},\mu)\left|\mathcal{H}_{t}\right.\right)\leq\min\left\{ \frac{5\rho_{k_t+1}(\mathcal{F}_{t})}{\rho_{k_t}(\mathcal{F}_{t})},1\right\} .\label{eq:5-2}
\end{equation*}
Note, that condition on $\mathcal{H}_t$, the only randomness in the above expression comes from $x$. Combining \cref{eq:0001} with the last two inequalities yields
\begin{equation*}
\begin{aligned}P_{t} & \leq1.2\epsilon+\mathbb{E}\left[\min\left\{ \frac{5\rho_{k_t+1}(\mathcal{F}_{t})}{\rho_{k_t}(\mathcal{F}_{t})},1\right\} \right]+2\mathbb{P}(\mathcal{E}^{c})\\
 & \leq1.2\epsilon+\mathbb{E}\left[\frac{5\rho_{k_t+1}(\mathcal{F}_{t})}{\rho_{k_t}(\mathcal{F}_{t})}\1(\mathcal{E})\right]+3\mathbb{P}(\mathcal{E}^{c})\\
 & \leq1.2\epsilon+\mathbb{E}\left[\frac{7.5\rho_{k_t+1}^{\mathcal{S}}(\mathcal{F}_{t},\mu)}{\rho_{k_t}^{\mathcal{S}}(\mathcal{F}_{t})}\right]+3\mathbb{P}(\mathcal{E}^{c})\leq17.7\epsilon.
\end{aligned}
\end{equation*}
where the third inequality follows from \cref{eq:3-claim1,eq:3-1}, and the last inequality follows by $\Pbb(\Ecal^c)\leq 3\epsilon$ and the definition of $k_t$. Consequently,
\[
\Ebb[\textsc{AbsErr}(\Tcal,z)]\leq \sum_{t=1}^T P_t\leq 17.7\epsilon T.
\]
Finally, we need to verify the assumption that $\Pbb(\Ecal^c)\leq 3\epsilon$ for oblivious and adaptive adversaries. Formally, our goal is to show that, with at most $3\epsilon$ probability of failure,
\begin{equation}\label{eq:difference_bounded_by_variance}
    \forall\tilde \Fcal\in\mathcal{G},k\in[d],\quad\left|\rho_{k}^{\Scal}(\tilde \Fcal)-\rho_{k}(\tilde \Fcal,\mu)\right|\leq 2\sigma_k^N(\tilde\Fcal).
\end{equation}
We will separately show \cref{eq:difference_bounded_by_variance} for oblivious and adaptive adversaries.

% Since the algorithm cannot directly access $\rho_k (\Fcal_t,\mu)$,  it has to make decisions based on $\rho_k ^\Scal(\Fcal_t)$. To ensure $\rho_{k+1}(\Fcal_t,\mu)/\rho_{k}(\Fcal_t,\mu)$ is sufficiently small, a simple analogy would be checking if 1. $\rho_{k+1} ^\Scal(\Fcal_t)/\rho_k ^\Scal(\Fcal_t)$ is small, and 2. $\rho_{k+1}(\Fcal_t,\mu)/\rho_{k}(\Fcal_t,\mu)\approx \rho_{k+1} ^\Scal(\Fcal_t)/\rho_k ^\Scal(\Fcal_t)$. 

% Since $\rho_k ^\Scal(\Fcal_t)$ is accessible to the algorithm, the first condition can be easily checked. The second condition is closely related to the standard deviation of estimation $\rho_{k} ^\Scal(\Fcal_t)$, which we will discuss later.
\paragraph{Verifying $\Pbb(\Ecal^c)\leq 3\epsilon$ for oblivious adversaries.}
When the adversary is oblivious, we may consider without loss of generality that the sequence of corrupted times as well as the corrupted samples are deterministic. Also, since in oblivious setting we let the algorithm update $\Fcal_t$ at every time $t>\max\Tcal$, we have $\Fcal_t=\Fcal\cap \{(x_\tau,y_\tau)\}_{\tau=\max\Tcal+1}^{t-1}$. 
Apart from the fixed corrupted samples, the sequence $(x_t)_{t=\max\mathcal{T}+1}^{T}$ is independent
of $(x_t)_{t\in\mathcal{T}}$ and as a result,
$\mathcal{G}$ is independent of $(x_t) _{t\in\mathcal{T}}$. Also recall that by definition, for any given function class $\tilde\Fcal$ and $k\in[d]$, the estimator $\hat\rho_k^{S_1}(\tilde\Fcal)$ is unbiaised for $\rho_k(\tilde\Fcal,\mu)$. Hence, using concentration bounds on the median in \cref{thm:median} with $\delta=\epsilon/dT$ and the union bound yields
\[
\Pbb(\Ecal^c)= \Ebb\sqb{\Pbb(\Ecal^c\mid\Gcal)} \leq\frac{\epsilon}{dT} |\mathcal{G}| \leq 3\epsilon.
\]

\paragraph{Verifying $\Pbb(\Ecal^c)\leq \epsilon$ for adaptive adversaries.}
For adaptive adversaries, the argument is more complex. This is because the samples $\Scal$ used for estimation (which are basically samples in $(x_t)_{t\in\Tcal}$) are potentially dependent on the sequence $(x_t)_{t=\max\Tcal+1}^T$, hence we can no longer assume samples in $\Scal$ are independent $\Gcal$. Despite not having this convenience, we can utilize \cref{lemma:uniform_convergence_median} which shows universal concentration over all possible function classes of the form $\Fcal\cap A$ where $\Fcal$ is the initial hypothesis class and $A$ can be any subset of less than $l$ data points in $\Xcal\times\{0,1\}$. 

For a run of $\textsc{WL}(\Tcal,z)$, define the set of times
\[\Mcal:=\{t\in (\max\Tcal,T] : \textsc{WL}(\Tcal,z)\text{ makes a mistake at time }t\}.\]
Since in adaptive setting we only let the algorithm update $\Fcal_t$ when there is a misclassification error, each function class $\tilde{\Fcal}\in\Gcal$ is of the form $\Fcal\cap A$ for some $A=\{(x_\tau,y_\tau):\tau\in\Mcal,\tau<t\}$, or $A=\{(x_\tau,y_\tau):\tau\in\Mcal,\tau<t\}\cap(x_t,z)$ for $z\in\{0,1\}$. Either way, $|A|\leq |\Mcal|+1$. 

From the analysis of Misclassification error, $|\Mcal|<5d^2\log(1/\epsilon)$ almost surely. Therefore, let $N=\ceil{2000d^2 /\epsilon}$, $\delta=\epsilon/d$ and $l=\ceil{5d^2 \log(1/\epsilon)}$ in \cref{lemma:uniform_convergence_median}. Then,
one can check the condition on $m$ for \cref{lemma:uniform_convergence_median} is satisfied: 
\[\ceil{c_0 (D\log(ND)+\log(1/\delta))}\leq \ceil{c_0(D\log(D)+8D+3\log(d/\epsilon))}\leq m.\]
Hence for each $k\in[d]$, with probabilty at least $1-\epsilon/d$,
    \begin{equation*}
        |\rho_k^\Scal(\Fcal\cap A) - \rho_k(\Fcal\cap A,\mu)| \leq 2 \sigma_k^N(\Fcal\cap A,\mu) ,\quad A\subseteq\Xcal\times\{0,1\}, \,|A|\leq \ceil{5d^2 \log(1/\epsilon)}.
    \end{equation*}
    In particular, every function class in $\Gcal$ takes the form $\Fcal\cap A$ for some set of data points $A$ with $|A|\leq \ceil{5d^2 \log(1/\epsilon)}$. Hence 
    \[\forall\tilde \Fcal\in\mathcal{G},\quad\left|\rho_{k}^{\Scal}(\tilde \Fcal)-\rho_{k}(\tilde \Fcal,\mu)\right|\leq 2\sigma_k^N(\tilde\Fcal).\]
    Union bounding over all $k\in[d]$ finishes the proof.
\end{proof}

Finally, we prove the weak learner guarantee for the censored variant model, which essentially follows the same proof.

\vspace{3pt}

\begin{proof}[of \cref{cor:iid_examples_censored}]
    By construction, $\textsc{WL}(\Tcal,\Ucal,z,u)$ makes the same updates for the function class $\Fcal_t$ as $\textsc{WL}(\Tcal,z)$ with $\texttt{update}=\texttt{restricted}$, and hence outputs the same values. Hence, it suffices to focus on proving the guarantees for $\textsc{WL}(\Tcal,z)$ with the same parameters and $\texttt{update}=\texttt{restricted}$. Using \cref{thm:iid_examples_v2} directly proves the desired result for the adaptive case. 
    
    For the oblivious case, similar to the proof of \cref{thm:iid_examples_v2}, it suffices to show $\Pbb(\Ecal^c)\leq 3\epsilon$ where 
    \begin{equation}
    \mathcal{E}\coloneqq\left\{ \forall\tilde \Fcal\in \{\Fcal_t,\Fcal_t ^{x_t\to 0},\Fcal_t ^{x_t\to 1}\}_{t=\max\Tcal+1}^T,k\in[d],\quad\left|\rho_{k}^{\Scal}(\tilde \Fcal)-\rho_{k}(\tilde \Fcal,\mu)\right|\leq 2\sigma_k^N(\tilde\Fcal)\right\}.
    \end{equation}
    
Since the adversary is oblivious, we consider without loss of generality that the sequence of corrupted times as well as the corrupted samples are deterministic. Recall that by the proof of \cref{thm:iid_examples_v2}, $\textsc{MisErr}(\Tcal,z)< 5d^2 \log(1/\epsilon)$ always holds.
Hence, for any fixed set of times $\Mcal\subseteq [T]$ with $|\Mcal|\leq l\coloneqq \floor{5d^2 \log(1/\epsilon)}$, define the set of function classes
\[\Gcal(\Mcal):=\set{\Fcal\cap \{(x_s,y_s)\}_{s\in\Mcal,s<t}:t\in (\max\Tcal,T]}.\]
Note that $\{\Fcal_t,\Fcal_t ^{x_t\to 0},\Fcal_t ^{x_t\to 1}\}_{t=\max\Tcal+1}^T = \Gcal(\Mcal)$ when $\Mcal$ corresponds to the set of times when $\textsc{WL}(\Tcal,z)$ makes a mistake. By construction, the weak learner does not make mistakes for $t\in [\max\Tcal]$. Therefore, we can consider
\begin{equation}
\Gcal_{\text{all}}=\bigcup_{
    \Mcal'\subseteq (\max\Tcal,T],
    |\Mcal'|\leq l
} \Gcal(\Mcal')
\end{equation}
which includes all function classes in $\{\Fcal_t,\Fcal_t ^{x_t\to 0},\Fcal_t ^{x_t\to 1}\}_{t=\max\Tcal+1}^T $ almost surely. 

Apart from the fixed corrupted samples, the sequence $(x_t)_{t=\max\mathcal{T}+1}^{T}$ is independent 
of $(x_t)_{t\in\mathcal{T}}$ and as a result,
$\mathcal{G}_{\text{all}}$ is independent of $(x_t) _{t\in\mathcal{T}}$.
Also recall that by definition, for any given function class $\tilde\Fcal$ and $k\in[d]$, the estimator $\hat\rho_k^{S_1}(\tilde\Fcal)$ is unbiaised for $\rho_k(\tilde\Fcal,\mu)$. Hence, using concentration bounds on the median in \cref{thm:median} with $\delta = (T^l (l+1)d)^{-1}\epsilon$---note that $m\geq \ceil{8 \log(1/\delta)}$ holds by assumption on $m$---and the union bound over all $\Fcal\in\Gcal_{\text{all}}$ yields
\begin{align*}
     \Pbb(\Ecal^c)&\leq \Pbb \left(\exists \tilde{\Fcal}\in\Gcal_{\text{all}},k\in[d],\quad\left|\rho_{k}^{\Scal}(\tilde \Fcal)-\rho_{k}(\tilde \Fcal,\mu)\right|> 2\sigma_k^N(\tilde\Fcal)\right)\\
     &\leq |\Gcal_{\text{all}}|d \delta \leq T^l \cdot 3(l+1)d\delta\leq 3\epsilon,
\end{align*}
where the second last inequality uses the observation that $|\Gcal(\Mcal)|\leq 3(|\Mcal|+1)$ for each $\Mcal$.
\end{proof}

\section{Proofs of \cref{sec:boosting}}

We start by proving the learning guarantee for the \textsc{Delete} algorithm which will be a subroutine within the complete boosting procedure.
\vspace{3pt}
\begin{proof}[of \cref{lemma:guarantee_delete_algo}]
    For each weak learner $i\in[L]$ we denote by $r_i=\abs{\set{t\in[n]\cap\Ucal: y_{i,t}\neq \perp}}$ its total number of predictions. Note that if $r_i=0$ for all $i\in[L]$ then weak learners always abstain and hence the desired result is immediate: the deletion algorithms also always abstain. We suppose this is not the case from now. We fix a deletion parameter $s\geq 0$ and denote by $N(s)=\abs{\set{t\in[n]\cap\Ucal: \hat y_t(s)\notin \{y_t,\perp\}}}$ the number of mistakes of $\textsc{Delete}_{s,\frac{C}{2}}$ on $\Ucal$.
    We consider the random variable
        \begin{equation*}
            r:=r_k \quad\text{where} \quad k\sim\text{Unif}(\{i\in[L]:r_i>0\})
        \end{equation*}
        is sampled uniformly among weak learners with at least one prediction. Next, for any $l\geq 1$ we define $q_l$ to be the $2^{-l}$-quantile of $r$, that is, $q_l\in\Nbb$ and
        \begin{equation*}
            \Pbb[r\geq q_l] \geq 2^{-l} > \Pbb[r>q_l].
        \end{equation*}
        We also pose $q_0=0$.
        Note that since there are at most $L$ rows, all quantiles $q_l$ are equal for $l\geq \log L$ to $r_{\max}:=\max_{i\in[L]} r_i$. 
        We first consider the case when $s\geq r_{\max}$. In that case, after deletions, we have $\tilde y_{a,b}=\perp$ for all $(a,b)\in[n]\times[L]$ hence $N(s)=0$. On the other hand, for $l\in[\ceil{\log L}]$, note that for any integer $s\in[q_{l-1},q_l)$, we have
        \begin{equation*}
            \Ebb[r-s\mid r> s] \geq (q_l-s)\Pbb[r\geq q_l\mid r>s] \geq (q_l-s)\frac{\Pbb[r\geq q_l]}{\Pbb[r>q_{l-1}]} \geq \frac{q_l-s}{2}.
        \end{equation*}
        In particular, this shows that
        \begin{align*}
            C|\Ucal| \overset{(i)}{\geq}\sum_{t\in[n]\cap\Ucal,i\in [L]} \1[\tilde y_{i,t}\neq\perp] &= \Ebb[r-s\mid r> s]\cdot  |\{i\in[L]:\exists t\in[n]\cap\Ucal, \tilde y_{i,t}\neq \perp\}|\\
            &\overset{(ii)}{\geq} \frac{q_l-s}{2M} \sum_{i\in [L]} \abs{\set{t\in[n]\cap\Ucal: \tilde y_{i,t} \notin\{y_t,\perp\}}}\\
            &\overset{(iii)}{\geq} \frac{q_l-s}{2M}\cdot \frac{C}{4}N(s)
        \end{align*}
        In $(i)$ we used the fact that at each round, at most $C$ weak learners make predictions.
        In $(ii)$, we used the previous lower bound on $\Ebb[r-s\mid r> s]$ as well as the assumption that each weak learner makes at most $M$ mistakes on $\Ucal$. In $(iii)$ we used the fact that by construction, $\textsc{Delete}_{s,C/2}$ follows the majority vote if at least $C/2$ weak learners make a prediction and otherwise abstains. Hence, if $\hat y_t(s)\notin\{y_t,\perp\}$ then at least $C/4$ weak learners must have made a mistake. Altogether, we showed that
        \begin{equation*}
            N(s) \leq \frac{8M|\Ucal|}{q_l-s}.
        \end{equation*}
        We recall that there are at most $\ceil{\log L}$ values of form $q_l$ for $l\in[\ceil{\log L}]$. In particular, either $s_{\max}+1\leq \ceil{\log L}$ in which case the desired bound is immediate since $N(s)\leq n$. Or $s_{\max}+1> \ceil{\log L}$ hence by the pigeonhole principle, either $s=r_{\max} \leq s_{\max}$ or there exists $s=q_{l-1}$ for $l\in[\ceil{\log L}]$ such that $q_l-s \geq (s_{\max}+1)/\ceil{\log L}$. In all cases, this proves the existence of $s\in \{0,\ldots,s_{\max}\}$ such that
        \begin{equation*}
            N(s) \leq \frac{8M|\Ucal| \ceil{\log L}}{s_{\max}+1}.
        \end{equation*}
        This ends the proof.
\end{proof}
We next prove the misclassification error of the \textsc{Aggregate} algorithm which combines the predictions of the $\textsc{Delete}_{s,C}$ algorithms for $s\in\{0,\ldots,s_{\max}\}$ using standard misclassification guarantees for WMA \citep{littlestone1994weighted}.

\vspace{3pt}

\begin{proof}[of \cref{lemma:guarantee_aggregate_algo}]
    The misclassification error bound for $\textsc{Aggregate}_{s_{\max},C/2}$ is essentially an application of standard bounds for WMA, e.g. \cite{littlestone1994weighted}, which gives $M_T(\mathrm{WMA}) \leq 3 (M^\star + \log N)$ for where $M_T(\mathrm{WMA})$ (resp. $M^\star$) is the number of mistakes of WMA (resp. of the best expert) after $T$ rounds, for $N$ experts. We apply this bound to the set $\Ucal$ of times when all $\textsc{Delete}_{s_{\max},C/2}$ makes a prediction, which implies that all algorithms $\textsc{Delete}_{s,C/2}$ for $s\in\{0,\ldots,s_{\max}\}$ also make a prediction---indeed, by construction, each algorithm $\textsc{Delete}_{s,C/2}$ for $s\in\{0,\ldots,s_{\max}\}$, deletes at most the first $s_{\max}$ predictions for each weak learner.
    Note that by construction, on all other times, $\textsc{Aggregate}_{s_{\max},C/2}$ abstains, and hence incurs no misclassification error. Hence, the misclassification error of the WMA aggregates satisfies
    \begin{align*}
        \abs{ \set{t\in[n]:  \hat y_t \notin \{y_t,\perp\} }} = \abs{ \set{t\in\Ucal:  \hat y_t\neq y_t }}
        &\leq 3 \min_{s\leq s_{\max}} \abs{ \set{t\in \Ucal:  \hat y_t(s) \notin \{y_t,\perp\} }}  + 3\log(s_{\max}+1).
    \end{align*}
    Furthering the right-hand side using \cref{lemma:guarantee_delete_algo} gives the desired misclassification error bound.
    %Next, by construction, each algorithm $\textsc{Delete}_{s,C/2}$ for $s\in\{0,\ldots,s_{\max}\}$, deletes at most the first $s_{\max}$ predictions for each weak learner. Hence, all such algorithms make a prediction if and only if $\textsc{Delete}_{s_{\max},C/2}$ makes a prediction, i.e., if at least $C/2$ weak learners $i\in[L]$ with $|\{ t'<t:y_{i,t'}\neq \perp\}|\geq s_{\max}$ make a prediction $y_{i,t}\neq \perp$. By construction, at such times, $\textsc{Aggregate}_{s_{\max},C/2}$ also makes a prediction, ending the proof.
\end{proof}

We are now ready to prove the main guarantee for our boosting procedure $\textsc{Boosting}$.

\vspace{3pt}

\begin{proof}[of \cref{thm:combining_weak_lerners}]
    Fix the parameters $\epsilon,s_{\max},M$. As in the algorithm, we denote by $y_{t,i}$ the prediction of the weak learner $i\in[L]$ at time $t\in[T]$ for the sequence of instances generated by the adversary. 
    
    \paragraph{Correctness of the algorithm.} We start by checking that \cref{alg:main_exp_version} is well-defined, specifically, that the while loop in lines 6-13 at each round $t\geq 1$ terminates. Indeed, fix any iteration of this loop, and let $j\in[\ceil{\log L}]$ be the index such that $n_t:=|\set{i\in[L]:z_{i,t}\neq \perp}|\in (2^{-j}L,2^{1-j}L]$. From \cref{lemma:guarantee_aggregate_algo}, the layer-$j$ subroutine $\textsc{Aggregate}_{s_{\max},2^{-j}L}^{\Qcal_j}$ makes a prediction $\hat y_t^{(j)}\neq \perp$ whenever at least $2^{-j}L$ weak learners $i\in[L]$ make a prediction $z_{i,t}\neq \perp$ and have made at least $s_{\max}$ predictions in previous inputs to the layer-$j$ subroutine. Additionally, note that the quantity $s_{\max}-s_{i,j}$ precisely counts the number of predictions of weak learner $i$ in previous inputs to this subroutine. Hence, after running lines 10-12, the prediction $z_{i,t}$ of weak learner $i\in[L]$ is only kept if it made at least $s_{\max}$ predictions in previous inputs to the layer-$j$ subroutine. For clarity, denote by $\tilde z_{i,t}$ the updated value of $z_{t,i}$ after this operation. Altogether, this shows that if $\hat y_t^{(j)}=\perp$, then
    \begin{equation*}
        \tilde n_t:=|\set{i\in[L]: \tilde z_{i,t} \neq \perp}| < 2^{-j}L.
    \end{equation*}
    In turn, this shows that the variable $j$ is strictly increasing at each loop iteration, and hence must terminate when either a prediction is made or $n_t=0$.

    \paragraph{Misclassification error.} In line 4 of \cref{alg:main_exp_version}, we delete predictions of all weak learners with at least $M$ mistakes. To avoid confusions, we denote by $z_{i,t}^{(0)}$ the corresponding updated recommendation of weak learner $i\in[L]$ at iteration $t\in[n]$:
    \begin{equation*}
        z_{i,t}^{(0)} := \begin{cases}
            \perp &\text{if } |\set{s<t: y_{i,s}\notin\{y_s,\perp\}}| \geq M,\\
            y_{i,t} &\text{otherwise}.
        \end{cases}
    \end{equation*}
    By construction, for each weak learner $i\in[L]$, the recommendations $z_{i,t}^{(0)}$ for $t\in[T]$ contain at most $M$ mistakes.

    For any layer $j=1,\ldots,\ceil{\log L}$, we denote by $\Qcal_{j,T}$ the final value of $\Qcal_j$ at the end of the algorithm, that is, $\Qcal_j$ corresponds to the set of times on which we ran the subroutine $\textsc{Aggregate}_{s_{\max},2^{-j}L}$. Among these, we denote by $\Ucal_j$ the set of times when the subroutine $\textsc{Aggregate}_{s_{\max},2^{-j}L}$ makes a prediction:
    \begin{equation*}
        \Ucal_j:=\set{t\in\Qcal_j: \hat y_t^{(j)}\neq \perp}.
    \end{equation*}
    Next, for any $t\in\Qcal_{j}$, denote by $(z_{i,t}^{(j)})_{i\in[L]}$ the weak learner predictions input to this subroutine at iteration $t$. Note that the layer-$j$ weak learner recommendations $z_{i,t}^{(j)}$ for $i\in[L]$ and $t\in\Qcal_{j,T}$ are obtained from their counterpart $z_{i,t}^{(0)}$ by deleting some predictions: $z_{i,t}^{(j)}\in \{z_{i,t}^{(0)},\perp\}$. Indeed, throughout \cref{alg:main_exp_version}, the only updates of the quantities $z_{t,i}$ are deletions, see line 11. In turn this shows that the input fed to the subroutine $\textsc{Aggregate}_{s_{\max},2^{-j}L}$ contain at most $M$ mistakes during all rounds, as per \cref{def:adversary_structure}, and as a result, at most $M$ mistakes on times $\Ucal_j$. By construction, these inputs also have at most $2^{1-j}L$ predictions per round---see line 7 of \cref{alg:main_exp_version}---as per \cref{def:adversary_structure}. Hence, \cref{lemma:guarantee_aggregate_algo} bounds the misclassification error of the layer-$j$ subroutine by
    \begin{equation}\label{eq:miss_err_subroutine}
        \sum_{t\in\Qcal_{j}} \1[\hat y_t^{(j)} \notin\{y_t,\perp\}]  \leq \frac{24 M |\Ucal_j| \ceil{\log L}}{s_{\max}+1} + 3\log(s_{\max}+1).
    \end{equation}

    Since at each iteration $t\in[T]$ we follow the prediction of one of the subroutines or abstain, we can bound the total misclassification error by the sum of misclassification error of the subroutines:
    \begin{align*}
        \sum_{t=1}^T \1[\hat y_t\notin\{y_t,\perp\}] &\leq \sum_{j=1}^{\ceil{\log L}} \sum_{t\in\Qcal_{j,T}} \1[\hat y_t^{(j)} \notin\{y_t,\perp\}]\\
        &\overset{(i)}{\leq} \frac{24 M \ceil{\log L}}{s_{\max}+1}\sum_{j=1}^{\ceil{\log L}} |\Ucal_j| + 3\log(s_{\max}+1) \ceil{\log L}\\
        &\overset{(ii)}{\leq} \frac{24 M \ceil{\log L}}{s_{\max}+1}T + 3\log(s_{\max}+1) \ceil{\log L},
    \end{align*}
    where in $(i)$ we used \cref{eq:miss_err_subroutine}, and in $(ii)$ we used the fact that all sets $\Ucal_j$ are disjoint for $j\leq [\ceil{\log L}]$. Indeed, by construction of the \textsc{Boosting} algorithm, whenever the first layer makes a prediction $t\in\Ucal_1$, we follow this prediction $\hat y_t=\hat y_t^{(1)}$ and move to the next time $t+1$: hence $t\notin\Ucal_j$ for $j>1$. Similarly, if $t\in\Ucal_2$, the boosting strategy follows its prediction and moves to the next time $t+1$: $t\notin\Ucal_j$ for $j>2$; and so on. Using $s_{\max}\leq T$ gives the desired bound for the misclassification error of \cref{alg:main_exp_version}.

    \paragraph{Abstention error.} By construction, the algorithm abstains $\hat y_t=0$ only if at the end of while loop in lines 6-13, all weak learner updated recommendations are abstentions: $n_t=|\set{i\in[L]: z_{i,t}\neq \perp}|=0$. We recall that these updated recommendations are obtained from $z_{i,t}^{(0)}$ by potentially deleting predictions in line 11. Note, however, that there are at most $s_{\max}$ deletions for each layer $j$ throughout the complete procedure, as depicted by the counts $s_{i,j}$.
    %Hence, for each weak learner $i\in[L]$, there are at most $s_{\max}\ceil{\log L}$ times $t\in[T]$ for which at the end of the while loop $\perp=z_{i,t}\neq z_{i,t}^{(0)}$---that is, the prediction of weak learner $i$ was deleted. 
    Formally, if we denote by $z_{i,t}$ its value at the end of the while loop, for each $i\in[L]$ we have
    \begin{equation*}
        |\{t\in[T]: \hat y_t=\perp \text{ and } z_{i,t}^{(0)}\neq\perp\}| \leq |\{t\in[T]: z_{i,t}=\perp \text{ and } z_{i,t}^{(0)}\neq\perp\}| \leq s_{\max}\ceil{\log L}.
    \end{equation*}
    Additionally, if weak learner $i$ makes strictly less than $M$ mistakes, then line 4 of \cref{alg:main_exp_version} never deletes its predictions and hence $y_{t,i}=z_{i,t}^{(0)}$. Together with the previous equation this implies
    \begin{equation*}
        \textsc{AbsErr} \leq \textsc{AbsErr}(i) + s_{\max}\ceil{\log L},
    \end{equation*}
    where $\textsc{AbsErr}(i)$ denotes the abstention error of weak learner $i$. This
    ends the proof.
\end{proof}

We now prove the boosting guarantee for the variant censored model.

\vspace{3pt}

\begin{proof}[of \cref{cor:combining_weak_lerners_censored}]
    The correctness and abstention error follow from the same arguments as in \cref{thm:combining_weak_lerners}. We use the same notations. For bounding the misclassification error, observe that after the updated line 4 of the censored-boosting strategy \cref{alg:main_exp_version}, the corresponding updated recommendations $z_{i,t}^{(0)}$ of weak learner $i\in[L]$ at time $t$, contain at most $M$ mistakes on times when the final procedure made a prediction, that is, when at least one of the $\textsc{Aggregate}$ subroutines for each layer makes a prediction. In particular, on each layer $j$, the input fed to that subroutine $\textsc{Aggregate}_{s_{\max},2^{-j}L}$ contains at most $M$ mistakes on times $\Ucal_j$ when it made a prediction, which is precisely the property needed to apply \cref{lemma:guarantee_aggregate_algo}. The rest of the proof of \cref{thm:combining_weak_lerners} therefore applies and gives the desired misclassification bound.
\end{proof}

\end{document}